\documentclass{article}

\usepackage{arxiv}

\usepackage[utf8]{inputenc} 
\usepackage[T1]{fontenc}    
\usepackage{hyperref}       
\usepackage{url}            
\usepackage{booktabs}       
\usepackage{amsfonts}       
\usepackage{nicefrac}       
\usepackage{microtype}      

\usepackage{graphicx}
\usepackage{amsmath,amsthm}
\usepackage{color}
\usepackage{algorithm}
\usepackage{multirow}
\usepackage{enumitem, hyperref}
\usepackage{booktabs}
\usepackage{amssymb}
\usepackage[mathscr]{euscript}
\usepackage{mathrsfs}
\usepackage[noend]{algpseudocode}
\usepackage{natbib}
\usepackage{delimset}

\newtheorem{theorem}{Theorem}
\newtheorem{proposition}{Proposition}
\newtheorem{remark}{Remark}
\newtheorem{lemma}{Lemma}
\newtheorem{corollary}{Corollary}

\title{TEC: Tensor Ensemble Classifier for Big Data}

\author{
  Peide Li  \\
  Department of Statistics and Probability\\
  Michigan State University\\
  East Lansing MI \\
  \texttt{lipeide@msu.edu} \\
   \And
  Rejaul Karim \\
  Department of Statistics and Probability\\
  Michigan State University\\
  East Lansing MI \\
  \texttt{karim@gmail.com} \\
   \AND
    Tapabrata Maiti \thanks{The research is partially supported by NSF-DMS 1945824 and NSF-DMS 1924724}\\
  Department of Statistics and Probability\\
  Michigan State University\\
  East Lansing MI \\
   \texttt{maiti@stt.msu.edu} \\
}

\begin{document}
\maketitle

\begin{abstract}
Tensor (multidimensional array) classification problem has become very popular in modern applications such as image recognition and high dimensional spatio-temporal data analysis. Support Tensor Machine (STM) classifier, which is extended from the support vector machine, takes Candecomp / Parafac (CP) form of tensor data as input and predicts the data labels. The distribution-free and statistically consistent properties of STM highlight its potential in successfully handling wide varieties of data applications. Training a STM can be computationally expensive with high-dimensional tensors. However, reducing the size of tensor with a random projection technique can reduce the computational time and cost, making it feasible to handle large size tensors on regular machines. We name an STM estimated with randomly projected tensor as Random Projection-based Support Tensor Machine (RPSTM). In this work, we propose a Tensor Ensemble Classifier (TEC), which aggregates multiple RPSTMs for big tensor classification. TEC utilizes the ensemble idea to minimize the excessive classification risk brought by random projection, providing statistically consistent predictions while taking the computational advantage of RPSTM. Since each RPSTM can be estimated independently, TEC can further take advantage of parallel computing techniques and be more computationally efficient. The theoretical and numerical results demonstrate the decent performance of TEC model in high-dimensional tensor classification problems. The model prediction is statistically consistent as its risk is shown to converge to the optimal Bayes risk. Besides, we highlight the trade-off between the computational cost and the prediction risk for TEC model. The method is validated by extensive simulation and a real data example. We prepare a python package for applying TEC, which is available at our GitHub.
\end{abstract}


\section{Introduction}
\label{intro}
With the advancement of information and engineering technology, modern-day data science problems often come with data in gigantic size and increased complexity. One popular technique of storing these data is multi-dimensional arrays, which preserves the data anatomy and multi-dimensional spatial and spatio-temporal correlations. Unfolding the data may lose structure along with other relevant information and face the curse of dimensionality. The structural complexity and high dimensionality pose many challenges to statisticians and computer scientists in data processing and analysis. The traditional machine learning and statistical models for high dimensional data are developed in vector spaces. Adopting these methods to the new types of structured data will require a vectorization, resulting in losing critical information such as multi-way correlation and spatio-temporal information. Besides, the vectorized data's huge dimension could be a nightmare for traditional high dimensional data models, especially for their computation complexity. To preserve the multi-way structure of data and perform efficient computation, research has begun in building  tensor data models and performing statistical analysis based on these models.

A tensor is a multi-dimensional or multi-way array, a generalization of a matrix in a multi-dimensional situation. With different definitions, tensors can be represented in different forms, such as Candecomp / Parafac (CP), Tucker, and Tensor Chain. These tensor representations are widely applied in both supervised and unsupervised machine learning problems for multi-dimensional data. \citet{kolda2009tensor} provides a comprehensive survey paper about CP and Tucker tensor decomposition and their psychometrics, neuroscience, latent factor detection, and image compression applications prediction, and classification. Among all these machine learning applications with tensors, we focus only on classification models in this work. Various classification models for tensor data have been studied in the last few years. Some of works such as \citet{tan2012logistic}, \citet{zhou2013tensor}, and \citet{pan2018covariate}, mostly in statistics literature, develop logistic regression and probabilistic linear discriminant analysis models. These models rely on the same principle assumption that tensor data confirm a normal distribution, which is too restrictive in real applications. Non-parametric tensor classifiers turn out to be more flexible as they do not have any distribution assumption. Examples of non-parametric tensor classifiers include \citet{tao2005supervised}, \citet{kotsia2012higher}, \citet{signoretto2014learning}, \citet{he2014dusk}, \citet{li2014multilinear}, and \citet{he2017kernelized}. These methods are more flexible and would have much better generalization ability in applications with tensor data that are sparse or with high auto-correlation modes.

Motivated by the CP tensor and tensor product space introduced in \citet{hackbusch2012tensor}, we propose a version of Support Tensor Machine (STM) using tensor kernel functions based on tensor cross-product norm in our previous work \citet{li2019universal}. According to the definition, CP tensors can be considered as elements  and the associated cross-product norm measures the magnitude of tensors in a tensor product space. Following this idea, we develop a tensor kernel function based on cross-product norm to measure the similarity between tensors. Kernel functions also represent the resulting separating plane. We call this model CP-STM.


We consider simplifying the computation in CP-STM by training a STM with randomly projected tensors in this work. The application of random projection in high-dimensional data analysis is motivated by the well celebrated Johnson Lindenstrauss Lemma (see e.g. \citet{dasgupta2003elementary}). The lemma says that for arbitrary $k > \frac{8 \log n}{\epsilon^2}$, $\epsilon \in (0, 1)$, there is a linear transformation $f : \mathbb{R}^{p} \rightarrow \mathbb{R}^{k}$ such that for any two vectors $\textbf{x}_{i}, \textbf{x}_{j} \in \mathbb{R}^{p}$, $p > k$:
\begin{equation*}
    (1 - \epsilon) ||\textbf{x}_{i} - \textbf{x}_{j}||^{2} \leqslant ||f(\textbf{x}_{i}) - f(\textbf{x}_{j})||^{2} \leqslant  (1 + \epsilon) ||\textbf{x}_{i} - \textbf{x}_{j}||^{2}
\end{equation*}
with large probability for all $i, j = 1,...,n$. The linear transformation $f$ preserves the pairwise Euclidean distances between these points. The random projection has been proven to be a decent dimension reduction technique in machine learning literature. \citet{durrant2013sharp} presents a Vapnik-Chervonenkis type bounds on the generalization error of a linear classifier trained on a single random projection. \citet{chen2014convergence} provides a convergence rate for the classification error of the support vector machine trained with a single random projection. \citet{cannings2017random} proves that random projection ensemble classifier can reduce the generalization error further. Their results hold several types of basic classifiers such as nearest neighboring and linear/quadratic discriminant analysis. Inspired by these current work, we propose an ensemble of RPSTM, STM trained with randomly projected tensors, for large-sized tensor classification. We call this method Tensor Classification Ensemble (TEC).

\textbf{Our contribution}: Our work alleviates the limitations of existing tensor approaches in handling big data classification problems. This paper introduces an innovative ensemble classification model coupled with random projection and support tensor machines to provide computationally fast and statistically consistent label predictions for ultra high-dimensional tensor data. Specifically, the contribution of this paper is threefold.
\begin{enumerate}
    \item We successfully adopt the well known random-projection technique into high dimensional tensor classification applications and provide an ensemble classifier that can handle extremely big-sized tensor data. The adoption of random projection makes it feasible to directly classify big tensor data on regular machines and save computational costs. We further aggregate multiple RPSTM to form our TEC classifier, which can be statistically consistent while remaining computationally efficient. Since the aggregated base classifiers are independent of each other, the model learning procedure can be accelerated in a parallel computing platform.
    
    \item Some theoretical results are established in order to validate the prediction consistency of our classification model. Unlike \citet{chen2014convergence} and \citet{cannings2017random}, we adopt the Johnson-Lindenstrauss lemma further for tensor data and show that the STM from \citet{li2019universal} can be estimated with randomly projected tensors. The average classification risk of the estimated model converges to the optimal Bayes risk under some specific conditions. Thus, the ensemble of several RPSTMs can have robust parameter estimation and provide strongly consistent label predictions. The results also highlight the trade-off between classification risk and dimension reduction created by random projections. As a result, one can take a balance between the computational cost and prediction accuracy in practice. 
    
    \item We provide an extensive numerical study with synthetic and real tensor data to reveal our ensemble classifier's decent performance. It performs better than the traditional methods such as linear discriminant analysis and random forest, and other tensor-based methods in applications like brain MRI classification and traffic image recognition. It can also handle large tensors generated from tensor CANDECOMP/PARAFAC (CP) models, which are widely applied in spatial-temporal data analysis. Besides, the computation is much faster for the TEC comparing with the existing methods. All these results indicate a great potential for the proposed TEC in big data and multi-modal data applications. The Python module of applying TEC is available at \url{https://github.com/PeterLiPeide/KSTM_Classification/tree/master}.
\end{enumerate}

\noindent
\textbf{Outline:} The paper is organized in the following order. Section \ref{sec:related} provides the current literature on STM and contrast with our CP-STM. Section \ref{sec:metho} introduces  the classification methodology. Section \ref{sec:model} provides an algorithm for model estimation as well as parameter selection. Section \ref{sec:statprop} contains the main result of our TEC model's prediction consistency. The simulation study  and real data analysis are provided in Section \ref{sec:simul} and \ref{sec:RealData}. We analyze traffic imaging data in the real data application. Section \ref{sec:concl} concludes the article with a brief discussion. The proof of our theoretical results are provided in the appendix.

\section{Related Work}
\label{sec:related}
Support Tensor Machine (STM) generalizes the idea of support vector machine model from vectors and provides a framework for learning a separating hyper-plane for tensor data. The optimal separating plane would be represented by support tensors analogous to the support vector machine's support vectors. 

Previous work about STM includes \citet{tao2005supervised}, \citet{kotsia2012higher}, \citet{he2014dusk}, and \citet{he2017kernelized}. \citet{tao2005supervised} models the separating plane as the inner product between tensor predictors and a tensor coefficient assumed to be a rank-one tensor. The inner product is defined as a mode-wise product, which is indeed is a multi-linear transformation. \citet{kotsia2012higher} defines the margin as the traces of unfolded tensor coefficients and constructs the separating plane using multi-linear transformation. Both these works are established with the tensor Tucker model. \citet{he2014dusk} developed their STM under the idea of multi-way feature maps. The multi-way feature maps result in a kernel function measuring the similarity between CP tensors. The separating plane is also constructed with this kernel function. \citet{he2017kernelized} combines the tensor CP decomposition and CP tensor STM as a multi-task learning problem, making it slightly different from other STM methods. 

Among all these existing STM models, only the model from \citet{he2014dusk} is similar to our proposed RPSTM. Except for random projection, both methods use kernel functions to measure the similarity between tensors. However, the kernel function we adopted would be more general, since it is originated from the tensor cross-product norm. We will highlight this difference in more detail in the latter part of this paper. Moreover, we have provided the universal approximating property for such tensor kernel functions in our recent work \citet{li2019universal}. This gives a theoretical guarantee to the generalization ability for our model.

\section{Methodology}
\label{sec:metho}
\subsection{Notations and Tensor Algebra}
We refer \citet{kolda2009tensor} for some standard tensor notations and operations used in this paper. Numbers and scalars are denoted by lowercase letters such as $x, y$. Vectors are denoted by boldface lowercase letters, e.g. $\mathbf{a}$. Matrices are denoted by boldface capital letters, e.g. $\mathbf{A}, \mathbf{B}$. Multi-dimensional tensors are the generalization of vector and matrix representations for higher-order data, which are denoted by boldface Euler script letters such as $\mathscr{X}, \mathscr{Y}$.

The order of a tensor is the number of different dimensions, also known as ways or modes. For example, a scalar can be regarded as an order zero tensor, a vector is an order one tensor, and a matrix is an order two tensor. In general, a tensor can have $d$ modes as long as $d$ is an integer.

In addition to the basic concepts, we need some operations for vectors and matrices to present our work, which is also referred from \citet{kolda2009tensor}. The first one is the outer product of vectors. This is the simplest and the motivating example for tensor product. Let $\mathbf{a} \in \mathbb{R}^{p}$ and $\mathbf{b} \in \mathbb{R}^{q}$ be two row vectors, the outer product of them is defined by 
$$\mathbf{a} \otimes \mathbf{b} = \begin{bmatrix}
 a_{1}b_{1} & a_{1}b_{2} & ... & a_{1}b_{q} \\ 
a_{2}b_{1} & a_{2}b_{2} & ... & a_{2}b_{q}\\ 
 ... & ... & ... & ...\\ 
a_{p}b_{1} & a_{p}b_{2} & ... &a_{p}b_{q} 
\end{bmatrix}$$
which is a $p \times q$ matrix. This matrix is identical to $\mathbf{a}^{T}\mathbf{b}$. We use $\otimes$ to denote both outer product for vectors and CP tensors in our work.

The last part we want to mention is the rank of a tensor. There are two different definitions for tensor rank that are popular in  literature. One is the CP rank and the other is the Tucker rank. \citet{kolda2009tensor}  provides details for both ranks. We only consider the CP rank in this paper. The definition of CP rank for a d-mode tensor $\mathscr{X}$ is the minimum integer $r > 0 $ such that the tensor can be written as 
\begin{equation}
\begin{split}
        \mathscr{X} &= \sum \limits _{k = 1}^{r}  \mathbf{x}_{k}^{(1)} \otimes \mathbf{x}_{k}^{(2)} ... \otimes \mathbf{x}_{k}^{(d)} \\ 
        & = [\mathbf{X}^{(1)},..., \mathbf{X}^{(d)}] \quad \quad \mathbf{X}^{(j)} \in \mathbb{R}^{I_{j}\times r}, j = 1,..,d
\end{split}
\label{equ:Trank-r}
\end{equation}
Equitation (\ref{equ:Trank-r}) is also called tensor rank-r / CP representation (see \citet{hackbusch2012tensor}). We can always have a CP representation for tensors with rank $r$. For any given tensor, CP decomposition can be estimated with Alternative Least Square (ALS) algorithm (see \citet{kolda2009tensor}). The support tensor machine in this paper is constructed with rank-r tensors in their CP representations. As a result, all the tensors mentioned in the rest of the article are assumed to have decomposition (\ref{equ:Trank-r}).


\subsection{CP Support Tensor Machine}

Let $T_{n} = \{ (\mathscr{X}_{1}, y_{1}), ...., (\mathscr{X}_{n}, y_{n})\}$ be the training set, where $\mathscr{X}_{i} \in \mathbb{R}^{I_{1} \times I_{2} ... \times I_{d}}$ are d-mode tensors, $y_{i}$ are labels. If we assume the training risk of a classifier $f \in \mathcal{H}$ is $\mathcal{R}_{T_{n}}(f) = \frac{1}{n} \sum \limits_{i = 1}^{n} L(f(\mathscr{X}_{i}), y_{i})$, where $L$ is a loss function for classification, the problem will be looking for a $f_{n}$ such that
\begin{equation*}
f_{n} = \{f: \mathcal{R}_{T_{n}}(f) = \min \mathcal{R}_{T_{n}}(f), f \in \mathcal{H} \}
\label{equ:cproblem}
\end{equation*}
The $\mathcal{H}$ is a functional space for the collection of functions mapping tensors into scalars. The support tensor machine in \citet{li2019universal} solves the classification problem above by optimizing the objective function 
\begin{equation}
	\underset{f}{\min} \quad \lambda ||f||_{K}^{2} + \frac{1}{n} \sum \limits_{i = 1}^{n}L(y_{i}, f(\mathscr{X}_{i}))
\label{equ:STM1}
\end{equation}
where $||f||_{K} = (\int_{x} |f(x)|^{2} dx)^{\frac{1}{2}}$ is the L2 norm of the function $f$. Note that $L$ takes raw output of $f(\mathscr{X})$ instead of predicted labels. Examples of such loss functions in classification include Hinge loss and its variants. The optimal solution is in the form of 
\begin{equation}
\label{equ:STM}
    f_{n} = \sum \limits_{i = 1}^{n} \alpha_{i} y_{i} 
    K(\mathscr{X}_{i},\mathscr{X})
    \end{equation}
The kernel function is defined as 
\begin{equation}
    K(\mathscr{X}, \mathscr{Y}) = \sum \limits_{k, l = 1}^{r} \prod_{j = 1}^{d} K^{(j)}(\mathbf{x}_{k}^{(j)}, \mathbf{y}_{l}^{(j)})
\label{equ:TK}
\end{equation}
where $\mathscr{X} = \sum \limits_{k = 1}^{r} \mathbf{x}_{k}^{(1)} \otimes...\otimes \mathbf{x}_{k}^{(d)}$ and $\mathscr{Y} = \sum \limits_{k = 1}^{r} \mathbf{y}_{k}^{(1)} \otimes...\otimes \mathbf{y}_{k}^{(d)}$ are d-mode tensors. The label is predicted with $\text{Sign}(f_{n})$. $K^{(j)}(\cdot, \cdot)$ are vector based kernels as $\mathbf{x}_{k}^{(j)}$ and $\mathbf{y}_{k}^{(j)}$ are vector components of the tensors.

This function (\ref{equ:TK}) is a general form of kernel function, as we do not specify the explicit form of $K^{(j)}(\cdot, \cdot)$. It is developed with the idea of cross norm defined in the tensor product space, see \citet{hackbusch2012tensor} and \citet{ryan2013introduction}. It is very natural to derive a tensor kernel function in the form of equation (\ref{equ:TK}) when the CP factors of the tensor is available. Thus, \citet{he2014dusk} and \citet{he2017kernelized} also used the similar formulation. However, different choice of kernels might have impacts on the convergence of classification error. This has not been discussed in their works. We show that by choosing a universal kernel functions for tensors (see, \citet{li2019universal}), the classifier (\ref{equ:STM}) is consistent.

 A direct application of the above procedure can be computationally difficult and expensive  since $\mathbf{x}_{k}^{(j)}$ and $\mathbf{y}_{k}^{(j)}$ may belong to  high-dimensional spaces. Also, as the dimension in each mode increases, it will be more difficult for the kernel function to scale and provide accurate similarity measurement between tensor components. Thus, we apply a random projection for tensors before calculating the kernels.

\subsection{Support Tensor Machine with Random Projection}
For a single tensor $\mathscr{X} \in \mathbb{R}^{I_{1} \times ... \times I_{d}}$ with CP representation, we define the random projection of tensor as:
 \begin{equation}
    \begin{split}
        \mathscr{A} \circ \mathscr{X} & = \sum \limits_{k,l = 1}^{r}\mathbf{A}^{(1)}_{l}\mathbf{x}^{(1)}_{k} \otimes \mathbf{A}^{(2)}_{l}\mathbf{x}^{(2)}_{k} ... \otimes \mathbf{A}^{(d)}_{l}\mathbf{x}^{(d)}_{k} \\
    & =[\mathbf{A}^{(1)}\mathbf{X}^{(1)},..., \mathbf{A}^{(d)}\mathbf{X}^{(d)}] \quad \quad \mathbf{X}^{(j)} \in \mathbb{R}^{I_{j}\times r}, j = 1,..,d
\end{split}
\label{Tprojection}
\end{equation}
where $\mathbf{A}_{k}^{(j)} \in \mathbb{R}^{P_{j} \times I_{j}}, j = 1,...d, k = 1,...,r$ are matrices whose elements are independently identically distributed normal variables. Each component of the tensor CP expansion $\mathbf{x}_{k}^{(j)} \in \mathbb{R}^{I_{j}}$ is projected to a lower dimensional space $\mathbb{R}^{P_{j}}$ via a linear projection. We use $T^{\mathbf{A}}_{n}$ to denote the training set with tensors randomly projected to a lower dimensional space. The support tensor machine model will then be trained on the projected data instead of the original data. In general, the estimation procedure is expected to be much more efficient and faster as we assume $P_{j}$s are much smaller than $I_{j}$s.

With projected tensor data $T^{\mathbf{A}}_{n}$ and projection matrices, the original STM model (\ref{equ:STM}) turns to be
\begin{equation}
    \begin{split}
             f_{A,n}(\mathscr{X}) & =  \sum \limits_{i = 1}^{n} \alpha_{i}y_{i} K(\mathscr{A} \circ \mathscr{X}_{i}, \mathscr{A} \circ\mathscr{X}  )
    \end{split}
    \label{equ:CSTM_nonopt}
 \end{equation}
 However, it is not an optimal of STM in the projected training data. The  optimal solution in projected data space now becomes
 \begin{equation}
    \begin{split}
             g_{n}(\tilde{\mathscr{X}}) & =  \sum \limits_{i = 1}^{n} \hat{\beta}_{i}y_{i} K(\mathscr{A} \circ \mathscr{X}_{i},  \mathscr{A} \circ \mathscr{X})
    \end{split}
    \label{equ:CSTM}
 \end{equation}
In binary classification problems, the decision rule is $sign\{ g_{n}(\mathscr{X})\}$, where $g_{n}$ is the estimated function in the form of (\ref{equ:CSTM}) with appropriate choices of kernel functions. Choice of kernel functions are discussed in \citet{li2019universal} in detail.  We consider using Gaussian RBF kernel in this work. We name the model (\ref{equ:CSTM}) as Random Projection based Support Tensor Machine (RPSTM).

\subsection{Tensor Ensemble Classifier (TEC)}
While random projection provides extra efficiency by transforming tensor components into lower dimension, there is no guarantee that the projected data will preserve the same margin for every single random projection. Thus, the robustness of model estimation with a single random projection is difficult to maintain. The proposition in  Section \ref{sec:statprop}  shows that the classifier estimated with projected tensor only retains its discriminative power on average over multiple random projections. As a result, we consider aggregating several RPSTMs as an Tensor Ensemble Classifier (TEC) in order to provide a statistically consistent, robust, and computationally efficient classifier. Let 
\begin{equation}
    \tau_{n, b}(\mathscr{X}) = \frac{1}{b} \sum \limits_{m = 1}^{b} \text{Sign} (g_{n, b}(\mathscr{X}))
\end{equation}
$b$ is the number of different RPSTM classifiers. The ensemble classifier is then defined as
\begin{equation}
    e_{n, b} (\mathscr{X}) = \left\{\begin{matrix}
 1 & \text{if} \quad \tau_{n, b}(\mathscr{X}) \geqslant \gamma \\ 
 -1 & \text{Otherwise}
\end{matrix}\right.
\label{equ:ensembleclassifier}
\end{equation}
which is a voter classifier. $\gamma$ is a thersholding parameter. In the next section, we present the algorithm of model estimation.

\section{Model Estimation}
\label{sec:model}
Since each base classifier in the ensemble model, TEC is independently estimated, we only provide details for learning a single RPSTM. We estimate our model by adopting the optimization procedure from \citet{chapelle2007training}. This is a gradient descent optimization procedure solving the problem in primal. It solves the support vector machine type of problems efficiently without transforming the problem into its dual form.

Let $K$ be the kernel function for tensors. Let $\mathbf{A}^{(j)}_{k}$ be the random projection matrices for each mode of tensors. Let $\mathbf{D}_{y}$ be a diagonal matrix with $i$ th diagonal element is $y_{i}$. Since $y_{i}^{2} = 1$, $\mathbf{D}_{y}$ is an orthogonal matrix. i.e. $\mathbf{D}_{y}^{T}\mathbf{D}_{y} = \mathbf{I}$. The explicit form of the classification function can be derived as
\begin{equation}
\begin{split}
        g_{n} (\mathscr{A} \circ \mathscr{X}) & = \sum \limits _{i = 1}^{n} \beta_{i} y_{i} \sum \limits_{k, l = 1}^{r}\prod_{j = 1}^{d}K^{(j)}(\mathbf{A}^{(j)}_{k} \mathbf{x}^{(j)}_{k,i}, \mathbf{A}^{(j)}_{l}\mathbf{x}^{(j)}_{l} )\\
        & = \mathbf{K}^{T}\mathbf{D}_{y}\mathbf{\beta}
\end{split}
\label{equ:STMSolu}
\end{equation}
where $\beta = (\beta_{1}, ..., \beta_{n})^{T}$ and 
\begin{equation*}
    \mathbf{K}(:, i) = (\sum \limits_{k,l = 1}^{r}\prod_{j = 1}^{d}K^{(j)}(\mathbf{A}^{(j)}_{k} \mathbf{x}^{(j)}_{k,i}, \mathbf{A}^{(j)}_{l}\mathbf{x}^{(j)}_{l, 1})..,\sum \limits_{k, l = 1}^{r}\prod_{j = 1}^{d}K^{(j)}(\mathbf{A}^{(j)}_{k} \mathbf{x}^{(j)}_{k,i}, \mathbf{A}^{(j)}_{l}\mathbf{x}^{(j)}_{l, n} ))^{T}
\end{equation*}
$r$ is the CP rank of tensors $\mathbf{x}^{(j)}_{ik}$ and $\mathbf{x}^{(j)}_{k}$ are components of tensor CP decomposition of the corresponding training and testing tensors. In real data application, we may need to perform a CP decomposition first if the data does not come in CP form. Plugging the solution (\ref{equ:STMSolu}) into the objective function (\ref{equ:STM1}), we  get 
\begin{equation*}
	\underset{\beta}{\min} \quad \lambda \beta^{T}\mathbf{D}_{y}\mathbf{K}\mathbf{D}_{y}\beta   + \frac{1}{n} \sum \limits_{i = 1}^{n}L(y_{i}, \mathbf{K}^{T}(:, i)\mathbf{D}_{y}\beta)
\end{equation*}
where $\mathbf{K}(:, i)$ is the ith column of matrix $\mathbf{K} = [K_{1}, ..., K_{n}]$. The derivative with respect to $\beta$ is 
\begin{equation}
	2 \lambda \mathbf{D}_{y}\mathbf{K}\mathbf{D}_{y}\beta +  \frac{1}{n} \sum \limits_{i = 1}^{n} \mathbf{K}(:, i)\mathbf{D}_{y}\frac{\partial L}{\partial \beta}
	\label{equ:STMDeri}
\end{equation}
let equation (\ref{equ:STMDeri}) to be zero and we can solve  $\beta$ with Newton method. In our application, we take squared hing loss which is $L(y, (\mathscr{A} \circ\mathscr{X}) = [\max(0, 1 - yg_{n}(\mathscr{A} \circ \mathscr{X}))]^{2}$. The explicit form of equation (\ref{equ:STMDeri}), which we denote as $\bigtriangledown$, is 
\begin{equation*}
    \bigtriangledown = 2 \mathbf{D}_{y}\mathbf{K}\mathbf{D}_{y}[\lambda \beta + \mathbf{I}_{s}(\mathbf{D}_{y}\mathbf{K}\beta - \mathbf{y})]
\end{equation*}
where $\mathbf{y}$ is the vector of labels in the training data. $\mathbf{I}_{s} = diag(S)$, where $S \in \mathbb{R}^{n \times 1}$, $S_{i} = \{0, 1\}$ such that $S_{i} = 1$ if $y_{i}\mathbf{K}[:, i]^{T}\mathbf{D}_{y}\hat{\beta} < 1$. The Hessian, which is the derivative of $\bigtriangledown$, is 
\begin{equation*}
    \mathbf{H} = 2\mathbf{D}_{y}\mathbf{K}\mathbf{D}_{y}(\lambda \mathbf{I} + \mathbf{I}_{s}\mathbf{D}_{y}\mathbf{y})
\end{equation*}
The Newton method says that we can update $\beta \leftarrow \beta - \mathbf{H}^{-1} \bigtriangledown$ at each iteration. Thus, we obtain the update rule at each iteration as 
\begin{equation}
    \hat{\beta} = (\lambda \mathbf{I} + \mathbf{I}_{s} \mathbf{D}_{y} \mathbf{K})^{-1}\mathbf{I}_{s}\mathbf{y}
    \label{equ:updaterule}
\end{equation}

The algorithm for training and prediction are described in algorithm \ref{alg:STM} and algorithm \ref{alg:STMPre} respectively:\\
\begin{algorithm}[h]
\caption{TEC Training}
\label{alg:STM}
\begin{algorithmic}[1]
\Procedure{TEC Train} {} 
\State \textbf{Input:} Training set $T_{n} = \{ \mathscr{X}_{i} \}$, $Y = (y_{1}, ..., y_{n}))^{T}$, kernel function $K$, tensor rank r, $\lambda$, $\zeta$, maxiter, number of ensemble $b$
\For{i = 1, 2,...n}
    \State $\mathscr{X}_{i} = [\mathbf{B}_{i}^{(1)}, ..., \mathbf{B}_{i}^{(d)}]$ \Comment{CP decomposition by ALS algorithm}
\EndFor
\State Create initial matrix $\mathbf{K} \in \mathbb{R}^{n \times n}$
\For{l = 1, ..., b}
    \State Generate random matrices $\mathbf{A}^{(j)}_{k}$
\For{i = 1,...,n}
    \For{m = 1,...,i}
        \State $\mathbf{K}_{i,m} = \sum \limits_{k, l =1}^{r} \prod_{j = 1}^{d}K(\mathbf{A}^{(j)}_{k} \mathbf{B}^{(j)}_{i}[:, k], \mathbf{A}^{(j)}_{l} \mathbf{B}^{(j)}_{m}[:, l])$ \Comment{Kernel values}
        \State $\mathbf{K}_{m, i} = \mathbf{K}_{i, m}$
    \EndFor
\EndFor
\State Create $\hat{\beta} = \mathbf{1}_{n\times1}, \beta = \mathbf{0}_{n\times1}$ \Comment{Initial Value, can be different}
\State Iteration = 0
\While{$||\hat{\beta} - \beta|| \geqslant \zeta$ \& Iteration $\leqslant$ maxiter }
    \State $\beta = \hat{\beta}$
    \State Find $S \in \mathbb{R}^{n \times 1}$.  $S_{i} = \{0, 1\}$ such that $S_{i} = 1$ if $y_{i}\mathbf{K}[:, i]^{T}\mathbf{D}_{y}\hat{\beta} < 1$ \Comment{Indicating support tensors}
    \State $\mathbf{I}_{s} = diag(S)$ \Comment{Create diagonal matrix with $S$ as diagonal}
    \State $\hat{\beta} = (\lambda \mathbf{I} + \mathbf{I}_{s} \mathbf{D}_{y} \mathbf{K})^{-1}\mathbf{I}_{s}\mathbf{y}$ \Comment{Update}
\EndWhile
    \State \textbf{Output:} $\hat{\beta}$
\EndFor
\EndProcedure
\end{algorithmic}
\end{algorithm}

\begin{algorithm}[h]
\caption{TEC Prediction}
\label{alg:STMPre}
\begin{algorithmic}[1]
\Procedure{TEC Predict} {} 
\State \textbf{Input:} Decomposed training tensors $T_{de} = \{[\mathbf{B}_{i}^{(1)}, ..., \mathbf{B}_{i}^{(d)}] \}$, kernel function $K$, tensor rank r, $\lambda$, $\hat{\beta}$, new observation $\mathscr{X}$, random matrices $\mathbf{A}^{(j)}_{k}$, $\gamma$
\State $\mathscr{X} = [\mathbf{B}^{(1)}, ..., \mathbf{B}^{(d)}]$ \Comment{CP decomposition for New observation}
\For{$l = 1,...,b$}
    \For{i = 1,...,n}
       \State $\mathbf{k}[i] = \sum \limits_{k, l =1}^{r} \prod_{j = 1}^{d}K(\mathbf{A}^{(j)}_{k} \mathbf{B}_{i}^{(j)}[:, k], \mathbf{A}^{(j)} \mathbf{B}^{(j)}[:, l])$ \Comment{Kernel values}
\EndFor
\State $g_{l} = \mathbf{k}^{T}\mathbf{D}_{y}\hat{\beta}$
\EndFor
\State If $\text{Avg}(g) >\gamma$, the prediction is class 1. Otherwise, the prediction is -1
\State \textbf{Output:} Prediction
\EndProcedure
\end{algorithmic}
\end{algorithm}
Note that we use some R language style notations in the algorithms to specify rows, columns of matrices, and elements of vectors. For example, $\mathbf{K}[:, i]$ stands for the i-th column of matrix $\mathbf{K}$. $\mathbf{k}[j]$ denotes the j-th element in the vector $\mathbf{k}$.

In algorithm \ref{alg:STM}, the time complexity of the training process is $O(n^{2}r^{2}\sum \limits_{j = 1}^{d}P_{j})$, which is  much smaller than the complexity of any vectorized method, $O(n^{2}\prod_{j = 1}^{d}P_{j})$. The complexity of random projection is $O(\sum \limits_{j = 1}^{d} P_{j}I_{j})$. Thus, a single RPSTM with random projection has a computational complexity $O(n^{2}r^{2}\sum \limits_{j = 1}^{d}P_{j} + P_{j}I_{j})$. It is still can be smaller than the pure SVM whose complexity is $O(n^{2}\prod_{j = 1}^{d}I_{j})$

Moreover, the choices of $P_{j}$ are free from the original tensor dimensions. Our theoretical analysis reveals that $P_{j}$ are related to sample size $n$ instead of dimension $I_{j}$. Thus, applications with large size tensors can enjoy the benefit of random projection. Notice that, although the ensemble classifier has to repeat the estimation for $b$ many times, these repetitions are independent from each other. Thus, we can make these processes running in parallel without adding extra computation time.

\subsection{Choice of Parameters}
We end this section by providing some empirical suggestions about the tuning parameters used in the estimation procedure. The number of ensemble classifiers, $b$, and the threshold parameter, $\gamma$, are chosen by cross-validation. We first let $\gamma = 0$, which is the middle of two labels, -1 and 1. Then we search $b$ in a reasonable range, between 2 to 20. The optimal $b$ is the one that provides the best classification model. In the next step, we fix $b$ and search $\gamma$ between 1 and -1 with step size to be 0.1, and find optimal value which has the best classification accuracy. 

The choice of random projection matrices is more complicated. Although we can generate random matrices from standard Gaussian distribution, the dimension of matrices is remain unclear. Our guideline, JL-lemma, only provides a lower bound for dimension, and is only for vector situation. As a result, we can only choose the dimension based on our intuition and cross-validation results in practice. Empirically, we choose the projection dimension $P_{j} \approx \text{int}(0.7 \times I_{j})$ for each mode.

\section{Statistical Property of TEC Ensemble}
\label{sec:statprop}
In this section, we quantify the uncertainty of model prediction and provide some convergence results supporting the generalization ability of our model. Before the discussion, we need to introduce few more notations and highlight their differences. 

Let $L$ be a loss function for classification problems. In general, the problem is searching for a function in $\mathcal{H}$ such that it can minimize the risk of miss-classification.
\begin{equation*}
    \begin{split}
        f^{*} &= \arg \min \mathcal{R}(f) = \arg\min \mathbb{E}_{(\mathcal{X} \times \mathcal{Y}) } L(y, f(\mathscr{X})) \\
        f^{*} & \in \{f: \mathcal{X} \rightarrow \mathcal{Y} | f \in \mathcal{H} \}
    \end{split}
    \label{Bayes}
\end{equation*}
$\mathcal{H}$ is the collection of all measurable functions that map tensors into scalars. $\mathcal{R}(f)$ is the classification risk of a specific decision rule $f$, which is defined as 
\begin{equation*}
    \mathcal{R}(f) = \mathbb{E}_{(\mathcal{X} \times \mathcal{Y}) } L(y, f(\mathscr{X})) = \int L(y, f(\mathscr{X})) d\mathbb{P}
\end{equation*}
When the optimal  decision rule is obtained, i.e. $f = f^{*}$, $\mathcal{R}(f^{*})$ becomes the minimum risk that a decision rule is able to reach over the data distribution $(\mathcal{X} \times \mathcal{Y})$. It is called Bayes risk in the literature, and will be denoted by $\mathcal{R}^{*}$ in the following contents. We say a classifier $f$ is statistically consistent if $\mathcal{R}(f) \rightarrow \mathcal{R}^{*}$.

With finite training data $T_{n}$, however, we are not able to estimate such an optimal decision rule. Under most circumstances, we try to minimize the regularized empirical risk by selecting rules from a pre-defined collection of classification functions. Such a procedure is called Empirical Risk Minimization (EMR). The empirical risk of a decision rule estimated in $T_{n}$ is defined by
\begin{equation*}
    \mathcal{R}_{T_n}(f) = \frac{1}{n} \sum \limits_{i = 1}^{n}L(y_{i}, f(\mathscr{X}_{i}))
\end{equation*}
where we assume the probability mass is uniformly distributed on each data point in the training set. Under the random projection of training data $T^{A}_{n}$, the empirical risk in the projected training set is 
\begin{equation*}
    \mathcal{R}_{T^{\mathbf{A}}_{n}}(g) = \frac{1}{n} \sum \limits_{i = 1}^{n}L(y_{i}, g(\mathscr{A} \circ \mathscr{X}_{i}))
\end{equation*}
where $g$ is our decision function. Let $\mathcal{F}$ be the reproducing kernel Hilbert space (RKHS) reproduced by tensor kernel functions. By minimizing the regularized empirical risk, we shall obtain two functions which are
\begin{equation*}
\begin{split}
    & f^{\lambda}_{n} = \arg \underset{f \in \mathcal{F}}{\min} \{ \mathcal{R}_{T_{n}}(f) + \lambda ||f||_{K}^{2} \} \\
    & g^{\lambda}_{n} = \arg \underset{g \in \mathcal{F}}{\min} \{\mathcal{R}_{T^{\mathbf{A}}_{n}}(g) + \lambda ||g||_{K}^{2} \}
\end{split}
\end{equation*}
$f_{n}$ and $g_{n}$ are optimal decision rules which are conditioning on the training data $T_{n}$ and $T^{\mathbf{A}}_{n}$. Ideally, we assume the best-in-class decision rules $f$ and $g$ are independent of training data, i.e.
\begin{equation*}
\begin{split}
    & f^{\lambda} = \arg \underset{f \in \mathcal{F}}{\min} \{ \mathcal{R}(f) + \lambda ||f||_{K}^{2} \} \\
    & g^{\lambda} = \arg \underset{g \in \mathcal{F}}{\min} \{ \mathcal{R}(g \circ \mathscr{A}) + \lambda ||g||_{K}^{2} \}
\end{split}
\end{equation*}
Lastly, suppressing notational dependence on $\lambda$, we denote our ensemble classifier as 
\begin{equation*}
   e_{n, b}(\mathscr{X}) = \mathbb{I} \{ \tau_{n, b}(\mathscr{X}) \geqslant \gamma \} = \mathbb{I} \{ \frac{1}{b} \sum \limits_{m = 1}^{b} \text{Sign} (g_{n, m}(\mathscr{X}) )\geqslant \gamma\}
\end{equation*}
where each $g_{n, m}$ is the optimal $g_{n}$ conditional on random projection matrices $\mathscr{A}_{m}$. $\mathbb{I}$ is an indicator function. 

Notice that both the RPSTM classifier $g_{n}$ and the TEC $e_{n, b}$ are randomized classifiers. Their performances as well as their prediction risks also depend on the random projection $\mathscr{A}$. As a result, we say these randomized classifiers are statistically consistent if $\mathbb{E}_{\mathbf{A}}[g_{n}] \rightarrow \mathcal{R}^{*}$ and $\mathbb{E}_{\mathbf{A}}[e_{n, b}] \rightarrow \mathcal{R}^{*}$.  $\mathbb{E}_{\mathbf{A}}$ denotes expectation over the distribution of random projection matrices.We establish these two results in  this section.

\subsection{Risk of Ensemble Classifier}
We first boud the expected risk of our TEC classifier $e_{n, b}$ by using the result from \citet{cannings2017random}, theorem 2.
\begin{theorem}
For each $b \in \mathbb{N}$,
\begin{equation*}\label{equ:newdecom}
    \mathbb{E}_{\mathbf{A}} [\mathcal{R}(e_{n, b})] - \mathcal{R}^{*} \leqslant \frac{1}{\min(\gamma, 1 - \gamma)} [\mathbb{E}_{\mathbf{A}}[\mathcal{R} (g_{n})] - \mathcal{R}^{*} ]
\end{equation*}
\end{theorem}
 This result says that the ensemble model TEC is statistically consistent as long as the base classifier, RPSTM, is consistent. Hence, we present an analysis in the following part to show the consistency of RPSTM and its convergence rates. 

\subsection{Excess Risk of Single Classifier}
As a statistically consistent model, we expect the excess risk of RPSTM, $\mathbb{E}_{\mathbf{A}}[\mathcal{R}(g^{\lambda}_{n})] - \mathcal{R}^{*}$, converges to zero. Let $D(\lambda)$ be a function of the tuning parameter $\lambda$ which is described under assumptions, a few paragraphs below. The following proposition about the excess risk provides the direction of convergence for model risk.
\begin{proposition}
    \label{prop:riskdecomp}
The excess risk is bounded above:
\begin{equation}
    \begin{split}
        \mathbb{E}_{\mathbf{A}}[\mathcal{R}(g_{n}^{\lambda})] - \mathcal{R}^{*}
        &\leqslant \,[\mathbb{E}_{\mathbf{A}}[\mathcal{R}(g^{\lambda}_{n}) -\mathcal{R}_{T^{\mathbf{A}}_{n}}(g^{\lambda}_{A,n})] + \mathbb{E}_{\mathbf{A}}[\mathcal{R}_{T_{n}^{\mathbf{A}}}(f^{\lambda}_{A,n}) - \mathcal{R}(f^{\lambda}_{A,n})]\\
        &  + [\mathcal{R}(f_{n}^{\lambda}) - \mathcal{R}_{T_{n}}(f_{n}^{\lambda})]  +[\mathcal{R}_{T_{n}}(f^{\lambda})-\mathcal{R}(f^{\lambda})]\\
        & + [\mathbb{E}_{\mathbf{A}}[\mathcal{R}(f^{\lambda}_{A,n})] +\lambda\|f^{\lambda}_{A,n}\|^{2}_{k} - \mathcal{R}(f^{\lambda}_{n})-\lambda \|f^{\lambda}_{n}\|^{2}_{k} ] + D(\lambda)
        \end{split}
\label{equ:riskbound}
\end{equation}
\end{proposition}
The bound is proved in appendix \ref{appen:riskdecomp}. The notations $g_{n}^{\lambda}$, $f_{n}^{\lambda}$, and $f_{A, n}^{\lambda}$ stand for function (\ref{equ:CSTM}), (\ref{equ:STM}), and (\ref{equ:CSTM_nonopt}) in section \ref{sec:metho}. All classifier super scripted by $\lambda$ denotes norm restricted classifier. For a given random projection, $g_{n}^{\lambda}$ represents norm restricted optimal svm classifier on projected training data. $f_{n}^{\lambda}$ represents norm restricted optimal svm classifier on original training data.For a given random projection characterized by projection matrix $A$, $f_{A,n}^{\lambda}$ represents a norm restricted classifier on projected training data, constructed directly from $f_{n}^{\lambda}$ using (\ref{equ:CSTM_nonopt}). $f^{\lambda}$ denotes the norm restricted oracle optimal svm classifier.  \\
As the left side of equation (\ref{equ:riskbound}) is non-negative, we only have to show that the right side of equation (\ref{equ:riskbound}) converges to zero. Bounded convergence theorem can be adopted then to provide $\mathbb{E}_{\mathbf{A}}[\mathcal{R}(g_{n})] - \mathcal{R}^{*} \rightarrow 0$. The convergence of excess risk and its rate demonstrate the generalization ability of the classifier. 

In the following part, we assume all the conditions listed below hold:
\begin{enumerate}[label=\textbf{AS.\arabic*}]
    \item \label{cond:A1} The loss function $L$ is $C(W)$ local Lipschitz continuous in the sense that for $|f_{1}| \leqslant W < \infty$ and $|f_{2}|\leqslant W < \infty$
    \begin{equation*}
        |L(y, f_{1}) - L(y, f_{2})| \leqslant C(W) |f_{1} - f_{2}| 
    \end{equation*}
    In addition, we need $\underset{ y}{\sup}L(y, 0) \leqslant L_{0} < \infty$.
    
    \item \label{cond:A2} The kernel function 
    \begin{equation*}
        K(\mathscr{X}_{1}, \mathscr{X}_{2}) = \sum \limits_{k, l = 1}^{r}\prod_{j = 1}^{d}K^{(j)}(\mathbf{x}^{(j)}_{1,k}, \mathbf{x}^{(j)}_{2,l})
    \end{equation*}
    satisfies the universal approximating property (see, \citet{li2019universal}).
    \item \label{cond:A3} Assume  $\sqrt{\sup K(\cdot, \cdot) }= K_{max} < \infty$.
    \item \label{cond:A4} For each component in the kernel function $K^{(j)}(a, b) = h^{(j)}(||a - b||^{2})$ or $h^{(j)}(\langle a,b \rangle)$. All of them are $L_{K}^{(j)}$-Lipschitz continuous
    \begin{equation*}
        |h^{(j)}(t_1) - h^{(j)}(t_2)| \leqslant L_{K}^{(j)} |t_{1} - t_{2}|
    \end{equation*}
    where $t_{1}, t_{2}$ are different. $L_{K} = \underset{j = 1, ..,d}{\max}L_{K}^{(j)}$.
    
    \item \label{cond:A5} All the random projection matrices $\mathbf{A}_{k}^{(j)}$ have their elements identically independently distributed as $\mathscr{N}(0, 1)$. The dimension of $\mathbf{A}_{k}^{(j)}$ is $P_{j} \times I_{j}$. For a $\delta_{1} \in (0, 1)$ and $\epsilon > 0$, assume for each fixed $j$, $ P_{j} = O(\frac{[\log \frac{n}{\delta_{1}}]^{\frac{1}{d}}}{\epsilon^{2}})$.
    
    \item \label{cond:A6} The hyper-parameter in the regularization term $\lambda = \lambda_{n}$ satisfies:
    \begin{equation*}
        \begin{split}
             \lambda_{n} \rightarrow 0 \quad &\text{as} \quad  n \rightarrow \infty\\
             n\lambda^{2}_{n} \rightarrow \infty \quad &\text{as} \quad  n \rightarrow \infty\\
              \end{split}
    \end{equation*}
    
    \item \label{cond:A7} For all the tensor data $\mathscr{X} = \sum \limits _{k = 1}^{r}  \mathbf{x}_{k}^{(1)} \otimes \mathbf{x}^{(2)} ... \otimes \mathbf{x}_{k}^{(d)} $, assume $||\mathbf{x}_{k}^{(j)}||^{2} \leqslant B_{x} < \infty$.
    \item \label{cond:A8} Bayes Risk remains unaltered for all randomly projected data. For all $\mathscr{A}$
    $$\mathcal{R}^{*}_{\mathscr{A}}=\mathcal{R}^{*} $$
    \item \label{cond:A9} Let $\mathcal{F}$ be reproducible Kernel Hilbert space (RKHS). We assume that, there always exists a function $f_{\lambda}$ in $\mathcal{F}$  that minimizes risk as well have bounded kernel norm. Equivalently, it minimizes Lagrangian which is sum of classification error i.e. $\mathcal{R}(f) - \mathcal{R}^{*}$ and  $\lambda$ characterized kernel norm penalty i.e  $\lambda ||f||_{K}^{2}$. Thus we define the minimum achievable norm penalized classification risk as  $D(\lambda)=\underset{f \in \mathcal{F}}{\min} \{ \mathcal{R}(f) - \mathcal{R}^{*} + \lambda ||f||_{K}^{2} \}$. This minimum error is attained by $f_{\lambda}$. We further assume that the relation between $D(\lambda)$ and $\lambda$ is given by $D(\lambda)=c_{\eta}\lambda^{\eta}$ with $0<\eta\leqslant 1$. We refer to definition 5.14 and section 5.4 in \citet{steinwart2008support} for further reading. 
    \item \label{cond:A10} The projection error ratio for each mode vanishes at rates depending on loss function.   $ \text{As} \quad  n \rightarrow \infty$
    \begin{equation*}
         \frac{\epsilon_{n}}{\lambda^{q}_{n}} \rightarrow 0 \quad
        \end{equation*}
        For hinge loss $q=1$ and square hinge loss $q=\frac{3}{2}$.

\end{enumerate}

The assumption \ref{cond:A1}, \ref{cond:A3}, \ref{cond:A6}, and \ref{cond:A7} are  commonly used in supervised learning problems with kernel tricks, (see, e.g. \citet{vapnik2013nature}, \citet{laconte2005support}). Assumption \ref{cond:A2} is a form of universal kernel which makes it possible to approximate any continuous function defined on a compact set with a function in the form of (\ref{equ:STM}). The universal approximating property for tensor kernels is first developed in \citet{li2019universal}. With universal tensor kernel function $K$, any continuous function $f$ can be approximated by $f_{n}$, where 
$$f_{n} = \sum \beta K(\mathscr{X}_{i}, \cdot)$$ 
over a $\|\|_{2}$ metric compact set $\mathbf{B}_{x} \subset \mathbb{R}^{I_{1} \times ... \times I_{d}}$. We further assume that our tensor features takes values in  $\mathbf{B}_{x}$ only. The approximation is bounded by $L_{\infty}$ norm  $$\underset{\mathscr{X} \in \mathbf{B}_{x}}{\sup} |f(\mathscr{X}) - f_{n}(\mathscr{X})| \leq \epsilon_{2}.$$ 
Assumption \ref{cond:A8}. This assumptions ensure that remains Bayes risk unchanged due to random projection.This assumption has also been made in \cite{cannings2017random}. It assumes that there is a random projection $\mathbf{A} \in \mathscr{A}$ such that $$\mathbb{P}(\{ \mathbf{x}: \eta(\mathbf{x}) \geq \frac{1}{2}\} \Delta \{ \mathbf{x}: \eta^{\mathbf{A}}(\mathbf{A} \mathbf{x}) \geq \frac{1}{2}\}) = 0$$ $A \Delta B = (A \cap B^{c}) \cup (A^{c} \cap B)$ denotes the symmetric difference between two sets $A$ and $B$. The notation $\eta$ in \citet{cannings2017random} represents Bayes classifier and is different form our notation used in \ref{cond:A9}. Their assumption states that the feature level set with the following property does not exist in probability. The property being, original Bayes decision half plane and Bayes decision half plane corresponding to their projection are different. Thus, the Bayes risk remains unchanged due to any random projections. \ref{cond:A5} is a sufficient condition that ensures that preservation of regular $L_{2}$ (Euclidean) distances  between two tensors and their corresponding projection. Such statement is known as concentration inequality. Same lemma for vector data is known as J-L lemma(\citet{dasgupta2003elementary}). It also highlights the way of selecting dimensions $P_{j}$ for random projections such that we can control projection error with certain threshold probability. Assumption \ref{cond:A9} refers to the rate of convergence of minimum risk attained by a function with bounded kernel norm, to Bayes risk. The rate is expressed in terms kernel norm bound. It also implies  universal consistency \ref{universal consistency}, \citet{chen2014convergence}. This rate is generally satisfied  for a broad class of data distribution, which assigns certain low density to data near Bayes risk decision boundary \citet{steinwart2008support}, Theorem 8.18. Assumption \ref{cond:A4} states that kernel inner product between two data points can be expressed as some of a Lipshtiz function of  distance between them. This assumptions connects kernel norm and regular $L_{2}$ norm defined on tensors. Hence this assumption is critical for establishment of approximate kernel norm isometry of projected tensors i.e. $|K(A\mathscr{X}_{1}, A\mathscr{X}_{2}) - K(\mathscr{X}_{1}, \mathscr{X}_{2})| \leqslant O(\epsilon^d)$ with large probability, for $d$ mode tensors. Assumption \ref{cond:A10} declares that the random projection vanishes with respect to kernel norm tuning parameter, as sample size increases. Its worth noting that the rate should differ according to loss function, \citet{chen2014convergence}.

\subsubsection{Price for Random Projection}
Applying random projection in the training procedure is indeed doing a trade-off between prediction accuracy and computational cost. Although the application of Johnson-Lindenstrauss lemma (e.g. see \citet{johnson1984extensions} and \citet{dasgupta1999elementary}) has indicated an approximate isometry for projected data in terms of kernel values, the decision rule may not be exactly same as the one estimated from the original data set. Thus, it is necessary to study the asymptotic difference of 
\begin{equation*}
  \mathbb{E}_{\mathbf{A}}\mathcal{R} (f_{A,n}) - \mathcal{R}(f_{n})
\end{equation*}
We first drop the expectation part, and present a result for $\mathcal{R} (f_{A,n}) - \mathcal{R}(f_{n})$. The convergence in expectation is established later. 
\begin{proposition}
Assume matrices $\mathbf{A}$ are generated independently and identically from a Gaussian distribution. With the assumptions \ref{cond:A1}, \ref{cond:A4}, \ref{cond:A5}, \ref{cond:A6}, and \ref{cond:A7}, for the $\epsilon^{d}$ described in \ref{cond:A5}. With probability $(1-2\delta_{1})$ and $q=1$ for hinge loss, and $q=\frac{3}{2}$ for square hinge loss function respectively.
\begin{equation*}
|\mathcal{R} (f_{A,n})+\lambda||f^{\lambda}_{A,n}||^{2}_{k} - \mathcal{R}(f_{n})-\lambda||f^{\lambda}_{n}||^{2}_{k} | =O(\frac{\epsilon^{d}}{\lambda^{q}})
\end{equation*}
where $n$ is the size of training set, $d$ is the number of modes of tensor.
\label{prop:rprisk}
\end{proposition}
The proof of this proposition is provided in the appendix \ref{append:propRPerror}. The value of $q$ depends on loss function as well kernel and geometric configuration of data, which is discussed in the appendix. This proposition highlights the trade-off between dimension reduction and prediction risk. As the reduced dimension $P_{j}$ is related to $\epsilon$ negatively, small $P_{j}$ can make the term converges at a very slow rate. 

\subsubsection{Convergence of Risk}
Now we provide a result which establishes the convergence for the risk of RPSTM classifier and reveals the rate of convergence.
\begin{theorem}[Convergence rate]\label{thm:pac}
 Given a training data $T_{n} = \{(\mathscr{X}_{1}, y_{1}),..,(\mathscr{X}_{n}, y_{n}) \}$. We can obtain a support tensor machine classifier $sign(\hat{g}(\cdot) )= sign(\sum \limits_{i = 1}^{n} \hat{\beta}_{i}\mathbf{K}(\mathbf{A} \circ \cdot, \mathscr{X}_{i}))$ by minimizing the regularized empirical risk (\ref{equ:STM1}). Let $\mathcal{R}^{*}$ be the Bayes risk of the problem, with all the assumptions \ref{cond:A1} - \ref{cond:A9}. For $\epsilon > 0$ such that for each $j=\{1,2,..d\}$ the projected dimension $P_{j}=\lceil3 r^{\frac{2}{d}} \epsilon^{-2} [log (n/\delta_{1})]^{\frac{1}{d}}\rceil+1$.  The generalization error of RPSTM is bounded with probability at least $(1 - 2\delta_{1}) (1 - \delta_{2})$, i.e., 
\begin{equation}
\begin{split}
&\mathcal{R}(g_{n}^{\lambda})- \mathcal{R}^{*} \leqslant  V(1)+V(2)+V(3)\\
&V(1)= 12 C(K_{max}\sqrt{\frac{L_{0}}{\lambda}})K_{max}\frac{ \sqrt{L_{0}}}{ \sqrt{n\lambda}} + 9 \tilde{\zeta}_{\lambda} \sqrt{\frac{\log(2/\delta_{2})}{2n}}+ 2 \zeta_{\lambda} \sqrt{\frac{2\log(2/\delta_{2})}{n}}\\
&V(2)=D(\lambda)\\
&V(3)= C_{d}\, \Psi \,[C(K_{max}\sqrt{\frac{L_{0}}{\lambda}})+\lambda \Psi]\,\epsilon^{d}
\end{split}
\end{equation}

where $\delta_{1} \in (0, \frac{1}{2}) ,\delta_{2} \in (0, 1)$
\label{consistencythm}
\end{theorem}
\label{discussion_epsilon}
The proof is provided in appendix \ref{append:TECrisk}.The probability $1 - 2\delta_{1}$ is contribution due to the random projection. We notice that in the aforementioned expression, risk difference upper bound depends on probability parameter $\delta_{1}$ and $\delta_{2}$. The dependence on $\delta_{1}$ is expressed through $\epsilon$. The projection error $\epsilon$, random projection probability parameter $\delta_{1}$ and projection dimension $P_{j}$ are related by equation $P_{j}=\lceil3 r^{\frac{2}{d}} \epsilon^{-2} [log (n/\delta_{1})]^{\frac{1}{d}}\rceil+1$. Hence $\epsilon^{d}=r \sqrt{log(\frac{n}{\delta_{1}})}/ \prod_{j=1}^{d}P_{j}$ depends on $\delta_{1}$ for any fixed $P_{j}$s corresponding to fixed sample size $n$. This expression can provide an alternate interpretation if one wants to consider mode wise projected dimensions $P_{j}$s as error parameters instead of mode wise projection error $\epsilon$.  In our case, error parameter is projection error,$\epsilon$. It is worth noting that $P_{j}$ grows with sample size $n$, due to the relation $P_{j}=\lceil3 r^{\frac{2}{d}} n^{\frac{\mu}{d}} [log (n/\delta_{1})]^{\frac{1}{d}}\rceil+1$ to facilitate the Assumption \ref{cond:A10}. Assuming that the projected dimensions $P_{j}$ to be uniform for all modes $j=1,\cdots, d$. Therefore, with probability at least $(1-2\delta_{1})(1-\delta_{2})$, the above risk difference bounded by a function of $\epsilon(\delta_{1}),\delta_{2}$. 

For different loss function, the convergence rate is different. Since we propose using squared hinge loss for model estimation in section \ref{sec:model}, the specific rate for squared hinge loss is presented here. However, we also provide the rate for ordinary hinge loss in the appendix \ref{append:ratehinge}.

\begin{proposition}
For square hinge loss, let $\epsilon=(\frac{1}{n})^{\frac{\mu}{2d}}$ for $0 < \mu< 1 $ and $\lambda=(\frac{1}{n})^\frac{\mu}{ 2\eta +3}$ for some $0<\eta\leqslant1$, $P_{j}=\lceil3 r^{\frac{2}{d}} n^{\frac{\mu}{d}} [log (n/\delta_{1})]^{\frac{1}{d}}\rceil+1$.    For some $\delta_{1} \in (0,\frac{1}{2})$ and $\delta_{2} \in (0,1)$ with probability $(1-\delta_{2})(1-2\delta_{1})$
\begin{equation}
 \mathcal{R}(g^{\lambda})- \mathcal{R}^{*} \leqslant C \sqrt{log(\frac{2}{\delta_2})} (\frac{1}{n})^{\frac{\mu\eta}{2\eta +3}}
\end{equation}
\end{proposition}
The rate of convergence is faster with increase in sample size, when high value of $\mu$ is  chosen.  For $\mu\to 1 $  the risk difference rate becomes  $(\frac{1}{n})^{\frac{d}{5}}$. The proof of this result is in appendix \ref{append:ratesqhinge}.

\begin{proposition}
For hinge loss,Let $\epsilon=(\frac{1}{n})^{\frac{\mu}{2d}}$ for $0 < \mu< 1 $ and $\lambda=(\frac{1}{n})^\frac{\mu}{ 2\eta +2}$ for some $0<\eta\leqslant 1 $,    $P_{j}=\lceil3 r^{\frac{2}{d}} n^{\frac{\mu}{d}} [log (n/\delta_{1})]^{\frac{1}{d}}\rceil+1$ , For some $\delta_{1} \in (0,\frac{1}{2})$ and $\delta_{2} \in (0,1)$ with probability $(1-2\delta_{1})(1-\delta_{2})$
\begin{equation}
 \mathcal{R}(g^{\lambda})- \mathcal{R}^{*} \leqslant C \sqrt{log(\frac{2}{\delta_2})} (\frac{1}{n})^{\frac{\mu\eta}{2\eta +2}}
\end{equation}

\end{proposition}
The rate of convergence is faster with increase in sample size, when high value of $\mu$ is  chosen. For $\mu\to 1 $  the risk difference rate becomes  $(\frac{1}{n})^{\frac{d}{4}}$. The proof of this result is in appendix \ref{append:ratehinge}.

\begin{theorem}
\label{thm:expectation}
The excess risk goes to zero as sample size increases, the $\mathbb{E}_{\mathbf{A}}$ denote expectation with respect to tensor random projection $\mathscr{A}$
\[ \mathbb{E}_{\mathbf{A}}[\mathcal{R}(g_{n}^{\lambda}) ]- \mathcal{R}^{*}  \to 0  \]

\end{theorem}
This is the expected risk convergence building on top of our previous results. The proof is provided in the appendix \ref{append:expectation}. This theorem concludes that the expected risk of RPSTM converges to the optimal Bayes risk. As a result, the RPSTM, as well as our ensemble model TEC, are statistically consistent. 


\section{Simulation Study}
\label{sec:simul}
In this section, we present an extensive simulation study with data generated by different models.  Several popular tensor based and traditional vector based classification algorithms are included in the study for comparison. For the traditional techniques, we consider linear discriminant analysis (LDA), random forest (RF) and support vector machine (SVM) (e.g., see \citet{friedman2001elements}). For the tensor-based classification, we select MPCA from \citet{Lu_MPCA_2008}, DGTDA from \citet{li2014multilinear}, and the CATCH model from \citet{pan2018covariate}. The classification error rate and computation time are presented for comparison. 

Synthetic tensors are generated from different types of tensor models introduced in \citet{kolda2009tensor}. All the models are described below in details. For simplicity, we use $\mathscr{X}_{1}$ and $\mathscr{X}_{2}$ to denotes data from two different class.  
\begin{enumerate}
    \item \textbf{F1 Model}: Low dimensional rank 1 tensor factor model with each components confirming the same distribution. Shape of tensors is $30 \times 30 \times 30$.
    \begin{equation*}
        \begin{split}
        & \mathscr{X}_{1} = \mathbf{x}^{(1)} \otimes \mathbf{x}^{(2)} \otimes \mathbf{x}^{(3)} \quad \mathbf{x}^{(j)} \sim \mathscr{N}(\mathbf{0}, I_{30}), j = 1,2,3\\
        & \mathscr{X}_{2} = \mathbf{x}^{(1)} \otimes \mathbf{x}^{(2)} \otimes \mathbf{x}^{(3)} \quad \mathbf{x}^{(j)} \sim \mathscr{N}(\mathbf{0.5}, I_{30}), j = 1,2,3
        \end{split}
    \end{equation*}\\
    
    \item \textbf{F2 Model:} High dimensional rank 1 tensor with normal distribution in each component. Shape of tensors is $50 \times 50 \times 50 \times 50$.
                        \begin{equation*}
                        \begin{split}
                               & \mathscr{X} = \mathbf{x}^{(1)} \otimes \mathbf{x}^{(2)} \otimes \mathbf{x}^{(3)} \otimes \mathbf{x}^{(4)} 
                              \quad \mathbf{x}^{(j)} \sim \mathscr{N}(\mathbf{0}, \mathbf{\Sigma}^{(j)}), j = 1,2,3,4\\
                               & \mathscr{X} = \mathbf{x}^{(1)} \otimes \mathbf{x}^{(2)} \otimes \mathbf{x}^{(3)} \otimes \mathbf{x}^{(4)} 
                              \quad \mathbf{x}^{(j)} \sim \mathscr{N}(\mathbf{1}, \mathbf{\Sigma}^{(j)}), j = 1,2,3,4\\
                              &\mathbf{\Sigma}^{(1)} = \mathbf{I}, \quad \mathbf{\Sigma}^{(2)} = \mathbf{\Sigma}^{(4)} = AR(0.7), \quad \mathbf{\Sigma}_{i,j}^{(3)} = \min(i, j)
                        \end{split}
                        \end{equation*}\\
    \item \textbf{F3 Model}: High dimensional rank 3 tensor factor model. Components confirm different Gaussian distribution. Shape of tensors is $50 \times 50 \times 50 \times 50$.
                        \begin{equation*}
                        \begin{split}
                               & \mathscr{X}_{1} = \sum \limits_{k = 1}^{3}  \mathbf{x}_{k}^{(1)} \otimes  \mathbf{x}_{k}^{(2)} \otimes  \mathbf{x}_{k}^{(3)} \otimes  \mathbf{x}_{k}^{(4)} 
                              \quad  \mathbf{x}_{k}^{(j)} \sim \mathscr{N}(\mathbf{0}, \mathbf{\Sigma}), j = 1,2,3,4\\
                              & \mathscr{X}_{2} = \sum \limits_{k = 1}^{3}  \mathbf{x}_{k}^{(1)} \otimes  \mathbf{x}_{k}^{(2)} \otimes  \mathbf{x}_{k}^{(3)} \otimes  \mathbf{x}_{k}^{(4)} 
                              \quad  \mathbf{x}_{k}^{(j)} \sim \mathscr{N}(\mathbf{1}, \Sigma), j = 1,2,3,4 \\
                               &\mathbf{\Sigma}^{(1)} = \mathbf{I}, \quad \mathbf{\Sigma}^{(2)} = \mathbf{\Sigma}^{(4)} = AR(0.7), \quad \mathbf{\Sigma}_{i,j}^{(3)} = \min(i, j)
                        \end{split}
                        \end{equation*}\\
    \item \textbf{F4 Model}: Low dimensional rank 1 tensor factor model with components confirming different distributions. Shape of tensor is $30 \times 30 \times 30$.
                        \begin{equation*}
                        \begin{split}
                            & \mathscr{X}_{1} = \mathbf{x}^{(1)} \otimes \mathbf{x}^{(2)} \otimes \mathbf{x}^{(3)} \quad  \mathbf{x}^{(1)} \sim \Gamma(4, 2), \mathbf{x}^{(2)} \sim \mathscr{N}(0, \mathbf{I}), \mathbf{x}^{(3)} \sim \mathbf{U}(0, 1) \\
                              & \mathscr{X}_{2} = \mathbf{x}^{(1)} \otimes \mathbf{x}^{(2)} \otimes \mathbf{x}^{(3)} \quad \mathbf{x}^{(1)} \sim \Gamma(6, 2), \quad \mathbf{x}^{(2)} \sim \mathscr{N}(0, \mathbf{I}), \mathbf{x}^{(3)} \sim \mathbf{U}(0,1)
                        \end{split}
                        \end{equation*}\\
    \item \textbf{F5 Model}: A higher dimensional version of F3 model. Tensors are having four modes with dimension $50 \times 50 \times 50 \times 50$
    \begin{equation*}
                        \begin{split}
                                \mathscr{X}_{1} = &\mathbf{x}_{1}^{(1)} \otimes \mathbf{x}^{(2)} \otimes \mathbf{x}^{(3)} \otimes \mathbf{x}^{(4)} \quad \mathbf{x}^{(1)} \sim \Gamma(4, 2), \mathbf{x}^{(2)} \sim \mathscr{N}(0, \mathbf{I})\\
                                &\mathbf{x}^{(3)} \sim \Gamma(2, 1), \mathbf{x}^{(4)} \sim \mathbf{U}(3.5, 4.5)\\
                               \mathscr{X}_{2} = & \mathbf{x}_{2}^{(1)} \otimes \mathbf{x}^{(2)} \otimes \mathbf{x}^{(3)} \otimes \mathbf{x}^{(4)} \quad \mathbf{x}^{(1)} \sim \Gamma(5, 2), \mathbf{x}^{(2)} \sim \mathscr{N}(0, \mathbf{I})\\
                                &\mathbf{x}^{(3)} \sim \Gamma(2, 1), \mathbf{x}^{(4)} \sim \mathbf{U}(4.5, 5.5)\\
                        \end{split}
                        \end{equation*}\\
    
    \item \textbf{M1 Model:} Tensor model with size $30 \times 30 \times 30$
    \begin{equation*}
        \begin{split}
            & \text{Vec}(\mathscr{X}_{1}) \sim \mathscr{N}(\mathbf{0}, \mathbf{I}_{27000})\\
            & \text{Vec}(\mathscr{X}_{2}) \sim \mathscr{N}(\mathbf{0.5}, \mathbf{I}_{27000})\
        \end{split}
    \end{equation*}\\
    
    \item \textbf{T1 Model}: A Tucker model. $\mathbf{Z}^{(1)}, \mathbf{Z}^{(2)}\in \mathbb{R}^{30 \times 30 \times 30}$ with elements independently and identically distributed. The correlation for rows in each mode of tensor is identity.
                     \begin{equation*}
                        \begin{split}
                                \mathscr{X}_{1} = &\mathbf{Z}^{(1)} \times_{1} \mathbf{\Sigma}^{(1)} \times_{2} \mathbf{\Sigma}^{(2)} \times_{3} \mathbf{\Sigma}^{(3)} \quad  \mathbf{Z}^{(1)} \sim \mathscr{N}(0, 1)\\
                              \mathscr{X}_{2} = &\mathbf{Z}^{(2)} \times_{1} \mathbf{\Sigma}^{(1)} \times_{2} \mathbf{\Sigma}^{(2)} \times_{3} \mathbf{\Sigma}^{(3)} \quad  \mathbf{Z}^{(2)} \sim \mathscr{N}(0.5, 1) \\
                              &\mathbf{\Sigma}^{(1)} = \mathbf{\Sigma}^{(2)} = \mathbf{I}, \quad \mathbf{\Sigma}^{(3)} = AR(0.7)
                        \end{split}
                        \end{equation*}
    
\end{enumerate}
Although our TEC classifier is proposed specially for CP tensor data, we still consider two general tensor models, \textbf{M1} and \textbf{T1}, in the simulation study. \textbf{M1} model generates random multi-dimensional arrays (tensors), and \textbf{T1} model generates Tucker tensors. The reason we include these two models is that CP tensor, or tensor CP decomposition, are not directly available in many applications. Some algorithms such as Alternating Least Square (ALS) has to be applied to obtain the CP tensor form. In models \textbf{F1} - \textbf{F5}, ALS should work very well and CP tensor can be well estimated. However, CP tensor may not be well estimated in model \textbf{M1} and \textbf{T1}. Thus, the simulation study with all these models can demonstrate how robust our TEC classifier is when CP tensor is not well approximated.

We generate 200 tensors from simulation models with a balanced design, i.e. each class has 100 tensors. We then randomly shuffle and select 140 tenors as training set, and validate classifiers over the rest 60 tensors. The evaluation metrics for classification performance are error rate and its standard error (S.E.). The error rate is the "0 - 1" loss between model prediction and ground truth. We split train and test data for 100 times to calculate the error rates. The error rates shown in the result table are the average error rates from all repetitions. The standard error is the standard deviation of error rates during the experiments, which measures the stability of model performance. The results are shown in table \ref{tab:Simu_Rest}. 
\begin{table}[!ht]
\centering
  \caption{ \label{tab:Simu_Rest} Simulation Results Comparison}
   \begin{tabular}{l l c c  c c c c c}
  \toprule
    Model & Methods & TEC  & SVM & LDA & RF & MPCA & DGTDA   & CATCH\\
  \midrule
    \multirow{2}{*}{\textbf{F1}} & Error Rate & \textbf{17.50} & 18.45      & 18.90  & 31.30 & 27.53  & 23.75   & 19.35  \\
                                   & S.E. & 3.72   &  6.93 &  5.35 &  4.33 & 5.65 & 8.32  &  3.26 \\
                                   & Time (s) & 0.4   & 0.56 & 1.04 & 0.56 & 0.35 & 0.4  & 1.22 \\
          \midrule
    \multirow{2}{*}{\textbf{F2}} & Error Rate &  \textbf{7.75}  & 14.65 & NA & 31.85 & 25.83 & 24.08 & NA \\
                                   & S.E. & 2.91   & 1.78 & NA & 4.37 & 5.87 & 4.98 & NA\\
                                   & Time (s) & 1.25   & 0.82 & NA & 0.5 & 0.65 & 0.67  & NA \\
          \midrule
    \multirow{2}{*}{\textbf{F3}} & Error Rate & \textbf{11.63}   & 19.75 & NA & 43.08 & 25.75 & 31.25  & NA\\
                                   & S.E. & 3.94  & 2.95 & NA & 4.35 & 8.73 & 7.68  & NA\\
                                   & Time & 13.95  & 78.24 & NA & 11.17 & 18.75 & 21.35  & NA\\
          \midrule
    \multirow{2}{*}{\textbf{F4}} & Error Rate & \textbf{27.95} &   28.20 & 47.90 & 30.15 & 37.56 & 33.95  & 33.75 \\
                                   & S.E. & 4.80   & 3.42 & 6.31 & 8.00 & 8.35 & 7.52  & 6.37\\
                                   & Time & 2.67 & 2.08 & 1.05 & 0.91 & 0.99 & 1.82  & 15.32  \\
          \midrule
    \multirow{2}{*}{\textbf{F5}} & Error Rate & \textbf{29.85} & 37.50 & NA & 37.55 & 35.83 & 34.08 & NA \\
                                   & S.E. & 2.73  & 3.35 & NA & 6.27 & 8.93 & 9.27 & NA\\
                                   & Time & 10.5  & 135.75 & NA & 13.81 & 20.66 & 36.36  & NA  \\
          \midrule
    \multirow{2}{*}{\textbf{M1}} & Error Rate & 2.32  & \textbf{1.25} & 13.18 & 2.71 & 5.06 & 4.08  & 3.35 \\
                                  & S.E. & 0.32   & 0.54 & 18.82 & 0.64 & 0.39 & 0.48 & 0.33\\
                                  & Time(s)& 0.55  & 0.54 & 0.50 & 0.50 & 0.24 & 0.26   & 1.03 \\
          \midrule
    \multirow{2}{*}{\textbf{T1}} & Error Rate & \textbf{4.07} & 4.32  & 20.85 & 5.55 & 18.78 & 15.36  & 9.25 \\
                                  & S.E. & 1.13 & 1.21 & 6.32 & 2.33 & 0.29 & 0.41  & 2.25\\
                                  & Time (s) &1.03  & 0.75 & 1.03 & 0.54 & 8.75 & 9.62  & 5.56\\
  \bottomrule
\end{tabular}

Simulations are done in a computer with a 12-core CPU and 32 GB memory. Error rates are in percentage. NA stands for no results available due to memory limit. 
\end{table}

We first notice that vector based LDA cannot classify high dimensional tensor data from models F2, F3, and F5 under the specific computation setup. The method requires computing inverse for high dimensional matrices, which takes lots of memory and cannot be done with the limited resources. CATCH model faces the same problems and could not provide any result. In this high dimensional tensor models, our proposed TEC has significantly better performance than other working methods. In  situations where tensors are in relatively low dimension, like F1 and F4, TEC is slightly better than other competitors. The performance as well as computational time are comparable. However, things change dramatically in high-dimensional cases. In F2 and F3, the error rates of TEC are almost less than half of the error rates of the second winner. Compared with support tensor machine, its computational time is ten times faster. In F5, TEC still has at least 5\% less error rates comparing with other competitors. The classification results in model \textbf{M1} and \textbf{T1} show that when the CP tensor cannot be well approximated by ALS algorithm, our TEC clasifier still have comparable performance. In the \textbf{M1}, which is a relatively simple problem, SVM wins our TEC with only 1 \%. However, when the data becomes complicated in \textbf{T1}, our TEC provides a slightly better performance than SVM and other competitors even through the CP tensor may not be approximated well.

\section{Real Data Analysis}
\label{sec:RealData}
In this section, we apply TEC to a real data application.

\subsection{KITTI Traffic Data Recognition}
The second application we conduct is traffic image data recognition. Traffic image data recognition is an important computer vision problem.  We considered the image data from the KITTI Vision Benchmark Suit. \citet{Geiger2012CVPR}, \citet{Fritsch2013ITSC}, and \citet{Geiger2013IJRR}  provided a detailed description and some preliminary studies about the data set. In this application, a 2D object detection task asks us to recognize different objects pointed out by bounding boxes in pictures captured by a camera on streets. There are various types of objects in the pictures, most of which are pedestrians and cars. We selected images containing only pedestrians or cars to test the performance of our classifier.

The first step we did before training our classifier is image pre-processing, which includes cropping the images and dividing them into different categories. We picked patterns indicated by bounding boxes from images and smoothed them into a uniform dimension $224 \times 224 \times 4$. Then every image is represented by a three modes tensor $\mathscr{X} \in \mathbb{R}^{224 \times 224 \times 4}$. There were 4487 car images and 28742 pedestrian images after clipping the data set. The whole data set is about 48 gigabytes in the form of ``float 64". Next, we evaluate their qualities and divide them into three groups, where each group has a different level of classification difficulties. Images having more than 40 pixels in height and fully visible will go to the easy group. Partly visible images having 25 pixels or more in height are in the moderate group. Those images which are difficult to see with bare eyes are going to the hard group. 

Since the numbers of images in each group are different, we have different designs for these groups. We randomly selected 2000 car images and 2000 pedestrian images for the training set for the easy group. The remaining images are used to test our fitted classification model. We have 500 pedestrian images and 2000 cars for training in the moderate group since pedestrian images are not enough for the previous design. These numbers are kept unchanged in the hard group. The average time cost for recognizing each image is about 1.3 seconds. We considered reducing the image tensors into a lower dimension $\mathbb{R}^{160 \times 160 \times 4}$ in this application. The width and height of images are compressed while the color layer information has been preserved under this setup. We compare our methods with publicly available outcomes in computer vision-related research and are selected from the KITTI benchmark website. We choose SubCNN, Shirt R-CNN, InNet, and CLA from \citet{xiang2017subcategory}, which use sub-category information and convolutional neural network to improve classification accuracy. THU CV-AI and DH-ARI are from \citet{yang2016exploit}, which included scale-dependent pooling and layer-wise cascaded rejection in the traditional CNN feature selection and classification. The result is presented in the table \ref{tab:KITTI}.

\begin{table*}[h]
    \centering
    \caption{Traffic Image Recognition Comparison}
    \begin{tabular}{l  c c c}
\toprule
    Method / Error Rate (\%) & Easy & Moderate & Hard\\
\midrule
    TEC & 10.03 $\pm$ 1.35 & 8.93 $\pm$ 1.26 & \textbf{4.39 $\pm$ 1.37}\\
    THU CV-AI  & 8.04 & 8.03 & 15.43\\
    DH-ARI & 9.13 & 8.52 & 17.75\\
    CLA & 9.49 & 11.01 & 25.5\\
    InNet & 9.74 & 11.05 & 20.54\\
    Shift R-CNN & 9.44 & 11.10 & 20.14 \\
    SubCNN & 9.25 & 11.16 & 20.76 \\
\bottomrule
    \end{tabular}

The results are in the form of "Error Rate $\pm$ Standard Error".
    \label{tab:KITTI}
\end{table*}

The TEC has comparable performance in easy and moderate groups, while it does much better in the hard group. Our classifier's error rate in the hard group is at least 5\% less than other available competitors. The improved performance is probably because the images in the hard group have higher resolutions. More pixels are in these images, and random projection efficiently reduces the data dimension and filters out noises. Although our method does not use any neural network architecture, it is still very competitive in classification accuracy. According to the ranking system in the KITTI 2D objects detection task, which is ordered by the classification accuracy in the moderate group, our classifier has a very high ranking. We like to emphasize that we have a mathematical foundation for our method, which has provided reliable performance, compared to existing neural networks based methods which only have propagation structures.

\section{Conclusion}
\label{sec:concl}
We have proposed a tensor ensemble classifier with the CP support tensor machine and random projection in this work. The proposed method can handle high-dimensional tensor classification problems much faster comparing with the existing regularization based methods. Thanks to the Johnson-Lindenstrauss lemma and its variants, we have shown that the proposed ensemble classifier has a converging classification risk and can provide consistent predictions under some specific conditions. Tests with various synthetic tensor models and real data applications show that the proposed TEC can provide optimistic predictions in most classification problems.

Our primary focus in this work is on the classification applications on high-dimensional multi-way data such as images. Support tensor ensemble turns out to be an efficient way of analyzing such data. However, model interpretation has not been considered here. The features in the projected space are not able to provide any information about variable importance. Alternative approaches are possible for constructing explainable tensor classification models, but they are out of this article's scope. Besides that, selection for the dimension (size) of projected tensor $P_{j}$s cannot be addressed well at this moment. Although our theoretical result points out the connection between the classification risk and $\min P_{j}$, discussion about how to set $P_{j}$ for each mode of tensor may have to be developed in the future. 

In conclusion, TEC offers a new option in tensor data analysis. The key features highlighted in work are that TEC can efficiently analyze high-dimensional tensor data without compromising the estimation robustness and classification risk. We anticipate that this method will play a role in future application areas such as neural imaging and multi-modal data analysis.

{\it Acknowledgements: The authors greatly appreciate comments from two anonymous referees that improve the paper significantly. The research is partially supported by NSF-DMS 1945824 and NSF-DMS 1924724}

\bibliographystyle{apalike}
\bibliography{references}

\newpage
\appendix 
\section{Appendix}
\subsection{Definition of parameters}
\label{poi}
The following notations will be used to present our proof.
\begin{itemize}
    \item $f_{**}^{\lambda}=\arg \underset{f \in \mathcal{F} }{\min} 
 \{\mathcal{R}(f) - \mathcal{R}^{*} + \lambda ||f||_{K}^{2} \}$
 
    \item $\zeta_{\lambda}=\sup\{|l(y,f^{**}_{\lambda}(\mathscr{X})|:\mathscr{X} \in \mathbb{X} ;y \in Y\}$ for any feasible solution $f_{\lambda}$
    
    \item $\tilde{\zeta}_{\lambda}=\sup\{|l(y,y')|:\mathscr{X} \in \mathbb{X} ;
|y'|\leq K_{max}\sqrt{\frac{L_{0}}{\lambda}}\}$

    \item $C_{d,r}=(2L_{k}B^{2})^{d}r^{2}$ The constant bound on random projection error
    
    \item $\Psi=\underset{\mathscr{X}_{i} \in \mathbb{X} }{sup}\{\sum |\beta_{i}|: f=\sum\beta_{i}K(\mathscr{X}_{i},.)\}$ The supremum of absolute sum of all coefficients over every possible function in Reproducing Kernel Hilbert Space (RKHS).
\end{itemize}

\begin{remark}[Bounds for hinge and square hinge loss]
\label{remark:functions}
These quantities can be evaluated for hinge and square loss functions as follows
\begin{enumerate}
\item $\lambda \|f_{**}^{\lambda}\|_{K}^{2} < \{ \mathcal{R}(f_{**}^{\lambda}) - \mathcal{R}^{*} + \lambda ||f_{**}^{\lambda}||_{K}^{2} \} =D(\lambda) $
\item $\|f^{\lambda}\|_{\infty}\leq K \|f^{\lambda}\|_{K} \leq K \sqrt{\frac{L_{0}}{\lambda}}\|$ using Reproducible property and CS inequality
\item For hinge loss,$L_{0}=1,\zeta_{\lambda}\leq 1+ K_{max}\sqrt{\frac{1}{\lambda}}, \tilde{\zeta}_{\lambda} \leq 1+K_{max}\sqrt{\frac{D(\lambda)}{\lambda}},C(K_{max}\sqrt{\frac{1}{\lambda}})=1$ 
\item For square hinge loss,$L_{0}=1,\zeta_{\lambda}\leq (1+ K_{max}\sqrt{\frac{1}{\lambda}})^{2}, \tilde{\zeta}_{\lambda}\leq 2(1+K \frac{D(\lambda)}{\lambda}),C(K_{max}\sqrt{\frac{1}{\lambda}})=2K_{max}\sqrt{\frac{1}{\lambda}}$ from dual formulation.
\item For hinge loss, $\Psi_{L1}\leq \frac{1}{\lambda}$ from dual
\item For square hinge loss, using lemma  $\Psi_{L2}\leq O( \frac{1}{\lambda })$ 
\end{enumerate}
\end{remark}

\subsection{Useful lemmas}
\begin{lemma}
Let $n^{+}$ and $n^{-}$ be training samples with label $+1$ and $-1$ respectively. Define  $\Psi=\underset{\mathscr{X}_{i} \in \mathbb{X} }{sup}\{\sum |\beta_{i}|: f=\sum\beta_{i}K(\mathscr{X}_{i},.)\}$\\ Supremum of absolute sum of all coefficients over every possible function in RKHS is given by  
\begin{equation*}
\Psi_{L2}\leqslant \frac{1}{C_{\mathbf{K}}+\frac{n\lambda}{4n^{+}n^{-}}}     
\end{equation*}
where $C_{\mathbf{k}}= \min\limits_{\mathbf{\beta}:\mathbf{\beta}^{T}\mathbf{1}=1}\mathbf{\beta}^{T}\mathbf{K}\mathbf{\beta}$ depends on kernel matrix $\mathbf{K}$
\end{lemma}

\begin{proof}
Assuming the matrix $\mathbf{K}$ is positive definite. The above bounds are consistent with \citet{doktorski2011l2}
\end{proof}

\begin{corollary}
For square hinge loss, supremum of sum of absolute coefficients is finite as sample size grows.
$\Psi_{L2}=O(\frac{1}{\lambda})$
\end{corollary}

\begin{proof}
Sum of eigenvalues of $\mathbf{K}$ is of order $O(n)$ since trace of $\mathbf{K}$ is of order $O(n)$ from the fact that $K(\mathscr{X}_{i},\mathscr{X}_{i})=O(1)$ due to assumption \ref{cond:A7}. Considering $\frac{4n^{+}n^{-}}{n}\leqslant 1$ from arithmetic mean and geometric mean inequality. Assuming that  $\mathbf{K}$ is positive definite $\Psi_{L2}=O(1)$ .This bound agrees with the bound of Thm 3.3 of \citet{chen2014convergence} of order $O(\frac{1}{\lambda})$ when, $C_{\mathbf{K}}$ can be 0 under the situation where $\mathbf{K}$ is not positive definite.
\end{proof}

\subsubsection{Discussion on $C_{\mathbf{K}}$ }
\label{Psil2}
Note that  depending on kernel and geometric configuration of data points.
This quantity influences error of projected classifier.The above bounds are consistent with \citet{doktorski2011l2}.\\
Assuming the data are from bounded domain, i.e $\|\mathscr{X}\|_{2} \leq C_{d,r}$ The gram matrix $\mathbf{K}$ can have minimum eigenvalue as positive so $C_\mathbf{K} > 0$ . The key idea is to divide the bounded domain into minimal increasing sequence of discs $D_{n}$ formed between rings of radius $R_{n-1}$ and $R_{n}$ such that  $\mathscr{X} \subseteq  \cup_{n=1}^{N}D_{n}$. So the diameter of $\mathscr{X} \leq 2 R_{N}$ assuming $\sum R^{2}_n < \infty$ and then count the number of points in each $D_{n}$. So some regularity conditions on distribution of data for each disc $D_{n}$ is necessary to evaluate bounds on eigenvalues of Gram matrix.
For unknown case,we can estimate the Gram matrix.
\citet{shawe2005eigenspectrum} discusses the regularization error in such estimation\\

\subsection{Proof of proposition \ref{prop:riskdecomp}}
\label{appen:riskdecomp}
\begin{equation*}
    \begin{split}
        \mathbb{E}_{\mathbf{A}}[\mathcal{R}(g_{n}^{\lambda})] - \mathcal{R}^{*} & =  
        [\mathbb{E}_{\mathbf{A}}[\mathcal{R}(g^{\lambda}_{n}) -\mathcal{R}_{T^{\mathbf{A}}_{n}}(g^{\lambda}_{n})]
        + [\mathbb{E}_{\mathbf{A}}[\mathcal{R}_{T^{\mathbf{A}}_{n}}(g^{\lambda}_{n}) -\mathcal{R}_{T_{n}^{\mathbf{A}}}(f^{\lambda}_{A,n})-\lambda \|f^{\lambda}_{A,n}\|^{2}_{k}]\\
        &+ \mathbb{E}_{\mathbf{A}}[\mathcal{R}_{T_{n}^{\mathbf{A}}}(f^{\lambda}_{A,n}) - \mathcal{R}(f^{\lambda}_{A,n})] +  [\mathbb{E}_{\mathbf{A}}[\mathcal{R}(f^{\lambda}_{A,n})+\lambda \|f^{\lambda}_{A,n}\|^{2}_{k}] - \mathcal{R}(f^{\lambda}_{n})-\lambda \|f^{\lambda}_{n}\|^{2}_{k} ] \\
        &+[\mathcal{R}(f_{n}^{\lambda}) - \mathcal{R}_{T_{n}}(f_{n}^{\lambda})] +  [\mathcal{R}_{T_{n}}(f^{\lambda}_{n})+\lambda \|f^{\lambda}_{n}\|^{2}_{k} - \mathcal{R}_{T_{n}}(f^{\lambda})-\lambda \|f^{\lambda}\|^{2}_{k}]\\ 
        & + [\mathcal{R}_{T_{n}}(f^{\lambda})-\mathcal{R}(f^{\lambda})]
        +[\mathcal{R}(f^{\lambda}) - \mathcal{R}^{*} +\lambda\|f^{\lambda}\|^{2}_{k}]\\
        & \leqslant \,[\mathbb{E}_{\mathbf{A}}[\mathcal{R}(g^{\lambda}_{n}) -\mathcal{R}_{T^{\mathbf{A}}_{n}}(g^{\lambda}_{A,n})]+ \mathbb{E}_{\mathbf{A}}[\mathcal{R}_{T_{n}^{\mathbf{A}}}(f^{\lambda}_{A,n}) - \mathcal{R}(f^{\lambda}_{A,n})] + [\mathcal{R}(f_{n}^{\lambda}) - \mathcal{R}_{T_{n}}(f_{n}^{\lambda})]\\
        & + [\mathcal{R}_{T_{n}}(f^{\lambda})-\mathcal{R}(f^{\lambda})] + [\mathbb{E}_{\mathbf{A}}[\mathcal{R}(f^{\lambda}_{A,n})] +\lambda\|f^{\lambda}_{A,n}\|^{2}_{k} - \mathcal{R}(f^{\lambda}_{n})-\lambda \|f^{\lambda}_{n}\|^{2}_{k} ]
        +D(\lambda)
        \end{split}
    \label{equ:errordecom}
\end{equation*}
The term $[\mathbb{E}_{\mathbf{A}}[\mathcal{R}_{T^{\mathbf{A}}_{n}}(g^{\lambda}_{n}) -\mathcal{R}_{T_{n}^{\mathbf{A}}}(f^{\lambda}_{A,n})-\lambda \|f^{\lambda}_{A,n}\|^{2}_{k}] \leqslant 0$ since by definition $g^{\lambda}_{n}$ is the optimal of STM in the projected data. The first term is non positive by definition of $g^{\lambda}$.The second term denotes sample error in projected data.The third term is non positive since $g_{n}$ defined as risk minimizer of projected data. The fourth term denote projection error plus RKHS norm.The fifth term is non positive as well.  The sixth term denote sample error in original data.  The last term is universal approximation error $D(\lambda )$.\\
The first four terms representing sample errors of different classifiers are collected under $V(1)$. The fifth term representing projection error term is denoted by $V(3)$. The last term $D(\lambda)$ representing universal approximation is denoted by $V(2)$

\subsection{Proof of proposition \ref{prop:rprisk}}
\label{append:propRPerror}
Johnson-Lindenstrauss lemma  gives concentration bound on the error introduced by random projection in a single mode.\\
(e.g. see \citet{johnson1984extensions} and \citet{dasgupta1999elementary})
\begin{lemma}
For each fixed mode $j\in \{1,2,..,d\}$ and any two vectors ${x}^{j}_{1},{x}^{j}_{2} \in \mathbb{R}^{I^{j}\times 1}$ among $n$ training vectors, with probability at least $ (1 - \delta_{1})$ and the random projection matrices described in \ref{cond:A5}, we have 
\begin{equation*}
    |\,\, (||\mathbf{A^{j}}\mathbf{x}^{j}_{1} - \mathbf{A^{j}}\mathbf{x}^{j}_{2}||_{2}^{2}) - (||\mathbf{x}^{j}_{1} - \mathbf{x}^{j}_{2}||_{2}^{2})\,\,| \leqslant \epsilon ||\mathbf{x}^{j}_{1} - \mathbf{x}^{j}_{2}||_{2}^{2}
\end{equation*}
\end{lemma}

\begin{proof}
The matrix $\mathbf{A}^{j} \in \mathbb{R}^{p^{j}\times I^{j}}$, where $p_{j} = O(\frac{\log\frac{n}{\delta_{1}}}{\epsilon^{2}})$.The error introduced by random projection in the $j$ th mode can be restricted by Johnson-Lindenstrauss lemma (e.g. see \citet{johnson1984extensions} and \citet{dasgupta1999elementary}). For each fixed $j\in \{1,2,..,d\}$ With probability at least $ (1 - \delta_{1})$ and the random projection matrices described in \ref{cond:A5} 
\end{proof}
This result cannot be directly applied for tensor random projection, since it is sum of  product of norms, we need more generalized concentration.
 
 \begin{lemma}\label{lemma:shudy}
 Consider a $2d$ degree  polynomial  of independent Centered  Gaussian or Rademacher random variables as  $Q_{2d}(Y)=Q_{2d}(y_{i},..,y_{d})$. Then for some $\epsilon>0$ and $C_{d}$ constant depending on $d$.
\begin{equation*}
    \mathbb{P}(|Q_{2d}(Y)-\mathbb{E}(Q_{2d}(Y))|>\epsilon^{d} ) \leqslant e^{2}\exp \bigg( -\big(\, \frac{\epsilon^{2d}}{C_{d} \text{Var}[Q_{2d}(Y)]}\big)^{\frac{1}{d}}\bigg)
\end{equation*}
 \end{lemma}
\begin{proof}
The proof can be found using hypercontractivity,\citet{janson1997gaussian} Thm 6.12 and Thm 6.7, this result is also mentioned in \citet{schudy2012concentration}
 \end{proof}

A tensor $\mathscr{X}\in \mathbb{R}^{I_{1}\times.\times I_{d}}$ is $r$ CP decomposable  if it can be expressed as following. $\mathscr{X} = \sum \limits_{k = 1}^{r} \mathbf{x}_{k}^{(1)} \otimes...\otimes \mathbf{x}_{k}^{(d)}$. 
 \begin{proposition} \label{prop:polyJL}
For any two tensors $\mathscr{X}_{1},\mathscr{X}_{2}$ among $n$ training points of r CP decomposable tensors, concentration bounds of polynomials can be derived below. Given $\epsilon>0$, $C_{d}$ constant depending on $d$, we have following JL type result.
 \begin{equation*} 
 \begin{split}
 \mathbb{P}(\sum \limits_{k, l = 1}^{r} \prod_{j = 1}^{d} |\mathbf{A}^{(j)}\mathbf{x}^{(j)}_{1k} - \mathbf{A}^{(j)}\mathbf{x}^{(j)}_{2l} ||_{2}^{2} -  \sum \limits_{k, l = 1}^{r} \prod_{j = 1}^{d}||\mathbf{x}^{(j)}_{1k} - \mathbf{x}^{(j)}_{2l} ||_{2}^{2} > \epsilon^{d} \sum \limits_{k, l = 1}^{r} \prod_{j = 1}^{d}||\mathbf{x}^{(j)}_{1k} - \mathbf{x}^{(j)}_{2l} ||_{2}^{2}) \\ \leqslant O(n^{2})\,e^{2}\exp(-\bigg( 3^{d}r^{4}\epsilon^{2d} \prod_{j}^{d}P^{(j)} \bigg)^{\frac{1}{d}}) 
 \end{split}
 \end{equation*}
 \end{proposition}
 \begin{proof}
 
It is known that variable for any $x^{j}\in \mathbb{R}^{q^{j}}$ $\frac{\|\mathbf{A}^{(j)}\mathbf{x}^{(j)}\|^{2}_{2}}{\|\mathbf{x}^{(j)}\|_{2}}$ follows Chi square of degree of freedom $p^{j}$ with variance 3 $p^{j}$. Using the fact that,$\sum \limits_{k, l = 1}^{r}\prod_{j = 1}^{d} \frac{ ||\mathbf{A}^{(j)}\mathbf{x}^{(j)}_{1k} - \mathbf{A}^{(j)}\mathbf{x}^{(j)}_{2l} ||_{2}^{2}}{||\mathbf{x}^{(j)}_{1k} - \mathbf{x}^{(j)}_{2l} ||_{2}^{2}}$
 the sum of identical random variables which are  product of  independent Chi Square  variables,thus  a polynomial of degree=2d of Gaussian random variables. $O(n^{2})$ arises considering  pairwise concentration among $n$ training  tensors, using union bounds we prove this result.  \\
 \end{proof}

Since our kernel product defined on tensors is explicitly expressed as sum of mode wise kernel product of mode vectors in CP. We derive the following relation for our kernel product defined on tensors

 \begin{corollary}

For any two tensors $\mathscr{X}_{1},\mathscr{X}_{2}$ among $n$ training points of r CP decomposable tensors, concentration bounds of polynomials can be derived below. Given $\epsilon>0$,$\delta$ depending on $\epsilon$ as given in proposition \ref{prop:polyJL} $C_{d, r}$ constant depending on $d$, we have following JL type result. i.e. for a given tensor kernel function $K(\cdot, \cdot)$,
$$\mathbb{P}(|K(\mathbf{A}\mathscr{X}_{1}, \mathbf{A}\mathscr{X}_{2}) - K(\mathscr{X}_{1}, \mathscr{X}_{2})|  \geqslant C_{d,r}\epsilon^{d}) \leqslant \delta_{1}$$
\end{corollary}

\begin{proof}

 \begin{align*}
        && |K(\mathbf{A}\mathscr{X}_{1}, \mathbf{A}\mathscr{X}_{2}) - K(\mathscr{X}_{1}, \mathscr{X}_{2})| 
        & \leqslant | \sum \limits_{k, l = 1}^{r} \prod_{j = 1}^{d} K^{(j)}(\mathbf{A}^{(j)}\mathbf{x}_{1k}^{(j)}, \mathbf{A}^{(j)}\mathbf{x}_{2l}^{(j)}) -  K^{(j)}(\mathbf{x}_{1k}^{(j)}, \mathbf{x}_{2l}^{(j)})| \\
        &&  \quad & \leqslant  \sum \limits_{k, l = 1}^{r} \prod_{j = 1}^{d} |K^{(j)}(\mathbf{A}^{(j)}\mathbf{x}_{1k}^{(j)}, \mathbf{A}^{(j)}\mathbf{x}_{2l}^{(j)}) -  K^{(j)}(\mathbf{x}_{1k}^{(j)}, \mathbf{x}_{2l}^{(j)})| \\
        && \text{(Lipschitz continuity)} \quad &  \leqslant  \sum \limits_{k, l = 1}^{r} \prod_{j = 1}^{d} L_{K}^{(j)}|\, ||\mathbf{A}^{(j)}\mathbf{x}^{(j)}_{1k} - \mathbf{A}^{(j)}\mathbf{x}^{(j)}_{2l} ||_{2}^{2} -  ||\mathbf{x}^{(j)}_{1k} - \mathbf{x}^{(j)}_{2l} ||_{2}^{2}|\\
        && \text{(Max $L^{(j)}_{k}$over j)} \quad &  \leqslant L_{K}^{d} \sum \limits_{k, l = 1}^{r} \prod_{j = 1}^{d} |\, ||\mathbf{A}^{(j)}\mathbf{x}^{(j)}_{1k} - \mathbf{A}^{(j)}\mathbf{x}^{(j)}_{2l} ||_{2}^{2} -  ||\mathbf{x}^{(j)}_{1k} - \mathbf{x}^{(j)}_{2l} ||_{2}^{2}|\\
        && \text{(Proposition\ref{prop:polyJL} )} \quad &  \leqslant \epsilon^{d}L_{K}^{d} \sum \limits_{k, l = 1}^{r} \prod_{j = 1}^{d} ||\mathbf{x}^{(j)}_{1k} - \mathbf{x}^{(j)}_{2k} ||_{2}^{2}\\
        && \text{(Bounded by diameter of $\mathscr{X}$ )} &  \leqslant  \sum \limits_{k, l = 1}^{r} 2^{d}\epsilon^{d} L^{d}_{K}B_{x}^{2d}\\
        && \text{} \quad &  \leqslant 2^{d} r^{2} L^{d}_{K}B_{x}^{2d}\epsilon^{d}\\
       && \text{(Denote $C_{d,r}=2^{d} r^{2} L^{d}_{K}B_{x}^{2d}$)} \quad & = C_{d,r}\epsilon^{d}
\end{align*}
Such part vanishes as $\epsilon^{d}$ becomes as small as possible. The probability holds with probability at least $(1 - \delta_{1})$ when conditions \ref{cond:A4}, \ref{cond:A5} and \ref{cond:A7} hold for tensor data.
\end{proof}

\subsection{Bounds on error due to projection}

We derive bounds for tensor data,for vector data \citet{chen2011new} provides such bound.
Now we include conditions \ref{cond:A1}, \ref{cond:A6} with the previous result.\\
\begin{proposition}\label{prop:randerr} With probability 
$(1-\delta_{1})$ the following expression is true
\begin{equation}
\begin{split}
    |\mathbb{E}_{\mathbf{A}}[\mathcal{R} (f^{\lambda}_{A,n})+\lambda\|f^{\lambda}_{A,n}\|^{2}_{k}] - \mathcal{R}(f^{\lambda}_{n})-\lambda\|f^{\lambda}_{n}\|^{2}_{k} | &\leqslant C_{d}\,\Psi \,[C(K_{max}\sqrt{\frac{L_{0}}{\lambda}})+\lambda \Psi]\,\epsilon^{d}\\
    &=O(\frac{\epsilon^{d}}{\lambda^{q}})
    \label{equ:exerror}
    \end{split}
\end{equation}
\end{proposition}
For all $\mathscr{A}$, with probability at least $(1-\delta_{1})$

\begin{equation*}
\nonumber
    \begin{split}
          &|\mathbb{E}_{\mathbf{A}}[\mathcal{R} (f^{\lambda}_{A,n})] - \mathcal{R}(f^{\lambda}_{n}) |  = |\mathbb{E}_{\mathbf{A}} \{\mathbb{E}_{(X, Y)} [ L(f^{\lambda}_{An}(x), y) - L(f^{\lambda}_{n}(x), y)]\}| \\
         & \leqslant L(K_{max}\sqrt{\frac{L_{0}}{\lambda}}) \cdot \mathbb{E}_{(X, Y)} \{ \mathbb{E}_{\mathbf{A}} |f^{\lambda}_{A,n}(x) - f^{\lambda}_{n}(x)| \} \\
        & =  C(K_{max}\sqrt{\frac{L_{0}}{\lambda}}) \sum \limits_{i = 1}^{n} \alpha_{i}\mathbb{E}_{(X, Y)} \{ |y_{i}| \cdot \mathbb{E}_{\mathbf{A}} [ |K(\mathbf{A}\mathscr{X}_{1}, \mathbf{A}\mathscr{X}_{2}) - K(\mathscr{X}_{1}, \mathscr{X}_{2})|] \}\\
         & \leqslant C(K_{max}\sqrt{\frac{L_{0}}{\lambda}}) \sum \limits_{i = 1}^{n} |\alpha_{i}|\cdot \mathbb{E}_{(X, Y)}  [ C_{d,r}\epsilon^{d} ]\\
         & \leqslant C(K_{max}\sqrt{\frac{L_{0}}{\lambda}})\Psi C_{d,r}\epsilon^{d}
    \end{split}
\end{equation*}
Similarly, For all $\mathscr{A}$, with probability at least $(1 - \delta_{1})$
\begin{equation*}
\begin{split}
\lambda\|f^{\lambda}_{A,n}\|^{2}_{k} - \lambda\|f^{\lambda}_{n}\|^{2}_{k} &\leqslant\lambda \sum \limits_{i = 1}^{n}\sum \limits_{j = 1}^{n} \alpha_{j} \alpha_{i} y_{i}y_{j} |K(\mathbf{A}\mathscr{X}_{1}, \mathbf{A}\mathscr{X}_{2}) - K(\mathscr{X}_{1}, \mathscr{X}_{2})|\\
&\leqslant \lambda (\sum \limits_{j = 1}^{n} |\alpha_{j}|)(\sum \limits_{j = 1}^{n} |\alpha_{j}|)C_{d,r}\epsilon^{d}\\
& \leqslant \lambda \Psi^{2} C_{d,r}\epsilon^{d}
\end{split}
\end{equation*}

Hence with probability at least $(1-\delta_{1})$
\begin{equation}
\begin{split}
    |\mathbb{E}_{\mathbf{A}}[\mathcal{R} (f^{\lambda}_{A,n})+\lambda\|f^{\lambda}_{A,n}\|^{2}_{k}] - \mathcal{R}(f^{\lambda}_{n})-\lambda\|f^{\lambda}_{n}\|^{2}_{k} | &\leqslant C_{d}\,\Psi \,[C(K_{max}\sqrt{\frac{L_{0}}{\lambda}})+\lambda \Psi]\,\epsilon^{d}\\
    &=O(\frac{\epsilon^{d}}{\lambda^{q}})
    \label{equ:exerror2}
    \end{split}
\end{equation}
Thus, using values of $\Psi$ , $C(K_{max}\sqrt{\frac{L_{0}}{\lambda}})$ from \ref{poi} for hinge loss,  we get $q=1$  and for square hinge loss ,we obtain  $q=\frac{3}{2}$ respectively. Thus the  result for proposition \ref{prop:rprisk}.

\subsection{Convergence with Rademacher Complexity: Bounds on sample error}
\label{append:radecomp}
In this subsection, we bound the sample errors of classifiers $g^{\lambda}_{n},f^{\lambda}_{A,n},f^{\lambda}_{n},f^{\lambda} $ are captured collectively as V(1). There are two terms in the right side of excess error (\ref{equ:newdecom}) can be bounded by introducing Rademacher complexity. The definition of Rademacher complexity (or Rademacher average) can be found in \citet{bartlett2002rademacher}. We use $\mathscr{R}_{n}(\mathscr{H})$ to denote the Rademacher complexity of the function class $\mathscr{H}$, and $\hat{\mathscr{R}}_{\mathbb{D}_{n}}$ to denote the corresponding sample estimate with respect to data $\mathbb{D}_{n}=\{z_{1},..,z_{n}\}$ where $z_{i}=(\mathscr{x}_{i},y_{i})$. We shall present two well established results about the Rademacher complexity without proof. One can find details about the proof from \citet{bartlett2002rademacher}.
\begin{theorem}\label{thm:talagrand inequality}
Consider a set of classifiers $\mathscr{H}=\{h:\|h\|_{\infty}\leq \tilde{\zeta}_{\lambda} \} $. $\forall \delta_2 > 0$, with probability at least $1 - \delta_2$, we obtain:
\begin{equation}
    \underset{h \in \mathscr{H}}{\sup} |\mathbb{E}[h(Z)] - \frac{1}{n}\sum \limits_{i = 1}^{n}h(Z_{i})| \leqslant 2\mathscr{R}_{n}(\mathscr{H}) + \tilde{\zeta}_{\lambda}  \sqrt{\frac{\log \frac{2}{\delta}}{2n}}
    \label{RadeC}
\end{equation}
\end{theorem}

The proof can also be found using McDiarmid's inequality. The Rademacher complexity gives out a way of measuring the richness of a functional space. The theorem make use of the complexity and provides a bound for the estimation error. With the definition of functional tensor product space, we propose a bound for tensor RKHS which is similar to the vector case.\\

\begin{proposition}
Let $g_{n}^{\lambda}$ be the function in projected data RKHS, that minimize $n$ th sample empirical risk as well as bounded kernel norm. then $\|g_{n}^{\lambda}\|_{\infty}\leqslant K \sqrt{\frac{L_{0}}{\lambda}}$
\end{proposition}

\begin{proof}
 Since $g^{\lambda}_{n} = \arg \underset{g \in \mathcal{F}}{\min} \{\mathcal{R}_{T^{\mathbf{A}}_{n}}(g) + \lambda ||g||_{K}^{2} \}$ we get 

\begin{equation*}
\begin{split}
   \lambda ||g^{\lambda}_{n}||_{K}^{2}\leqslant & \mathcal{R}_{T^{\mathbf{A}}_{n}}(\mathbf{0}) + \lambda ||\mathbf{0}||_{K}^{2}\\
  & =  \frac{\mathcal{R}_{T^{\mathbf{A}}_{n}}(\mathbf{0})}{\lambda}\\
   & \leqslant\frac{L_{0}}{\lambda}
\end{split}    
\end{equation*}
By RKHS property , for any function $g$ which is an element of RKHS
 \begin{align*}
g(x)= & \langle g,K(,x) \rangle\\
&\leq \|g\|_{K}\,\sqrt{\|K(x,),K(x,)\|_{K}} &&\text{CS inequality}\\
 \end{align*}
Taking supremum over $x$ on both sides we obtain the  result
\end{proof}

\begin{proposition}
Let $\mathscr{G}^{\lambda}=\{g:\|g\|_{\infty} \leqslant K \sqrt{\frac{L_{0}}{\lambda}}\}$ Assume conditions \ref{cond:A1}, \ref{cond:A3},\ref{cond:A4} and \ref{cond:A7}. Let $\delta_{2} > 0$, for all $A$ with probability at least $1 - \delta_{2}$
\begin{equation}
    |\mathcal{R}(g^{\lambda}_{n}) - \mathcal{R}_{T^{A}_{n}}((g^{\lambda}_{n})| \leqslant  4 C(K_{max}\sqrt{\frac{L_{0}}{\lambda}})\hat{\mathscr{R}}_{n}(\mathscr{G}^{\lambda}) + 3 \tilde{\zeta}_{\lambda} \sqrt{\frac{\log(2/\delta_{2})}{2n}}
    \label{sample3}
\end{equation}
\end{proposition}

\begin{proof}
Let $h(Z)=L(y,g(\mathscr{X}))$. Let $\mathscr{H}^{\lambda}=\{h:h=L(y,g(\mathscr{X}),g\in \mathscr{G}^{\lambda}\}$.Therefore by definition,$\|h\|_{\infty}\leqslant \tilde{\zeta}_{\lambda}$  We apply theorem \ref{RadeC} 

\begin{equation*}
\begin{split}
    \underset{g \in \mathscr{G}^{\lambda}}{\sup} |\mathcal{R}(g) - \mathcal{R}_{T^{A}_{n}}(g)| &\leqslant  2\mathscr{R}_{n}(\mathscr{H}^{\lambda}) +  \tilde{\zeta}_{\lambda} \sqrt{\frac{\log(2/\delta_{2})}{2n}}\\
    & \leqslant  2(\hat{\mathscr{R}}_{n}(\mathscr{H}^{\lambda})+  \tilde{\zeta}_{\lambda} \sqrt{\frac{\log(2/\delta_{2})}{2n}}) +  \tilde{\zeta}_{\lambda} \sqrt{\frac{\log(2/\delta_{2})}{2n}} \\
    &\text{(McDiarmid's inequality)}\\
    & \leqslant 2(\,2C(K_{max}\sqrt{\frac{L_{0}}{\lambda}})\,\hat{\mathscr{R}}_{n}(\mathscr{G}^{\lambda})+  \tilde{\zeta}_{\lambda} \sqrt{\frac{\log(2/\delta_{2})}{2n}}) +  \tilde{\zeta}_{\lambda} \sqrt{\frac{\log(2/\delta_{2})}{2n}} \\
    &\text{(Ledoux Talagrand contraction Inequality)}
    \label{sample2}
    \end{split}
\end{equation*}
\end{proof}
We have used contraction Inequality \citet{ledoux2013probability} applicable for Lipschitz function $h=L(y,g)$. Thus we have sharper inequality than that from the  \citet{chen2014convergence} since we do not have extra term $\frac{L_{0}}{\sqrt{n}}$.

\begin{proposition}
We provide upper bound on sample Rademacher complexity 
\[\hat{\mathscr{R}}_{n}(\mathscr{G}^{\lambda}) \leqslant\frac{K \sqrt{\frac{L_{0}}{\lambda}}}{\sqrt{n}}\]
\label{RKHSBound}
\end{proposition}

\begin{proof}
Let $\sigma$ be Rademacher variables with $\mathbb{P}_{\sigma}(\sigma_{i}=\pm 1)=\frac{1}{2} $
\begin{equation}
\begin{split}
 \hat{\mathscr{R}}_{n}(\mathscr{G}^{\lambda})&=\mathbb{E}_{\sigma}[ \underset{g \in \mathscr{G}^{\lambda}}{\sup} \frac{1}{n} \sum_{i=1}^{n} \sigma_{i}l(y_{i},g(A\circ \mathscr{X}_{i}))]\\
 &\leqslant \frac{\sqrt{L_{0}}}{n \sqrt{\lambda}} \sqrt{\sup\limits_{\mathscr{X}}\sum_{i}K(\mathscr{A}\circ\mathscr{X}_{i},\mathscr{A}\circ\mathscr{X}_{i})}\\
 &\leq \frac{K \sqrt{L_{0}}}{ \sqrt{n\lambda}}
\end{split}
\end{equation}
\end{proof}

Observing the above terms, we now bound individual error terms.

\begin{proposition}
With probability at least $(1-\delta_2)$ for $\delta_2\in (0,1)$ \[|\mathcal{R}_{T_{n}}(f^{\lambda})-\mathcal{R}(f^{\lambda})|< 2 \zeta_{\lambda} \sqrt{\frac{2log \frac{1}{\delta_2}}{n}} \]
\end{proposition}

Since loss function is bounded by $\zeta_{\lambda}$ for $\|f^{\lambda}\|_{\infty}=K_{max}\sqrt{\frac{L_{0}}{\lambda}}$,using Hoeffding's inequality,we obtain
\[\mathbb{P}[\mathcal{R}_{T_{n}}(f^{\lambda})-\mathcal{R}(f^{\lambda})] > \theta)\leq exp(-\frac{n\theta}{8\zeta^{2}_{\lambda}})\]
Choosing $\delta_2=exp(-\frac{n\theta}{8\zeta^{2}_{\lambda}})$ leads to the above bound. The same technique can be applied to bound $\mathbb{E}_{\mathbf{A}}[\mathcal{R}_{T_{n}^{\mathbf{A}}}(f^{\lambda}_{A,n}) - \mathcal{R}(f^{\lambda}_{A,n})]$ and $[\mathcal{R}(f_{n}^{\lambda}) - \mathcal{R}_{T_{n}}(f_{n}^{\lambda})]$, with almost exactly the same result. Thus, we skip the part. Now we are ready to combine all the results and prove the last theorem.

\begin{proposition}\label{prop:V(1)}
Assume conditions \ref{cond:A1} - \ref{cond:A7}. Let $\delta_{2} > 0$, for all projection matrix $\mathscr{A}$, with probability at least $1 - \delta_{2}$. The sample error which is denoted by V(1), can be bounded as follows
\begin{align*} \mathcal{R}_{T^{A}_{n}}(g^{\lambda}_{n}) - \mathcal{R}(g^{\lambda}_{n})  +  \mathcal{R}_{T_{n}^{\mathbf{A}}}(f^{\lambda}_{A,n}) - \mathcal{R}(f^{\lambda}_{A,n}) +  \mathcal{R}_{T_{n}}(f_{n}^{\lambda}) -\mathcal{R}(f_{n}^{\lambda})
        + \mathcal{R}_{T_{n}}(f^{\lambda})-\mathcal{R}(f^{\lambda})\\ \leq 12 C(K_{max}\sqrt{\frac{L_{0}}{\lambda}})\frac{K \sqrt{L_{0}}}{ \sqrt{n\lambda}} + 9 \tilde{\zeta}_{\lambda} \sqrt{\frac{\log(2/\delta_{2})}{2n}}+ 2 \zeta_{\lambda} \sqrt{\frac{2log \frac{2}{\delta_2}}{n}}
\end{align*}        
\end{proposition}

\begin{proof}
We bound the given risk difference by collecting the bounds on individual components from theorem \ref{thm:talagrand inequality} and proposition 5, proposition 6, proposition 7 and proposition 8.
\end{proof}

\begin{proposition}
\label{universal consistency}
The universal approximation rate for some $0 < \eta\leq 1$ 
\[D(\lambda)\leqslant C_{\eta} \lambda^{\eta}\]
\end{proposition}
Under universal consistent kernel. We assume this condition \ref{cond:A9}.This result holds for large class of data distribution satisfying low density near Bayes risk boundary.\\
\citet{steinwart2008support}Thm 8.18

\subsection{Proof of theorem \ref{consistencythm}}
\label{append:TECrisk}
Assume all the conditions \ref{cond:A1} - \ref{cond:A9} hold. The excess risk is decomposed as (\ref{equ:errordecom}). Since choosing random projection matrices and getting training data $T_{n}$ are independent events, by including proposition \ref{universal consistency}, for $\epsilon > 0$ such that for each $j=\{1,2,..d\}$ the projected dimension $P_{j}=\lfloor 3 r^{\frac{2}{d}}\frac{(log\frac{n}{\delta_{1}})^{\frac{1}{d}}}{\epsilon^{2}}\rfloor +1 $ Then we have with probability at least $(1 - 2\delta_{1})) (1 - \delta_{2})$ 
\begin{equation*}
\begin{split}
&\mathcal{R}_{\mathbf{A}}(g^{\lambda})- \mathcal{R}^{*} \leqslant  V(1)+V(2)+V(3)\\
&V(1)= 12 C(K_{max}\sqrt{\frac{L_{0}}{\lambda}})\frac{K \sqrt{L_{0}}}{ \sqrt{n\lambda}} + 9 \tilde{\zeta}_{\lambda} \sqrt{\frac{\log(2/\delta_{2})}{2n}}+ 2 \zeta_{\lambda} \sqrt{\frac{2\log(2/\delta_{2})}{n}}\\
&V(2)=D(\lambda)\\
&V(3)= C_{d}\, \Psi \,[C(K_{max}\sqrt{\frac{L_{0}}{\lambda}})+\lambda \Psi] \epsilon^{d}
\end{split}
\end{equation*}

\subsection{Rates for square hinge loss}
\label{append:ratesqhinge}
\begin{proposition}
For square hinge loss, for each $j=\{1,2,..d\}$ the projected dimension $P_{j}=\lfloor 3 r^{\frac{2}{d}}\frac{(log\frac{n}{\delta_{1}})^{\frac{1}{d}}}{\epsilon^{2}}\rfloor +1 $. Then we have with probability at least $(1 - 2\delta_{1})) (1 - \delta_{2})$, for some $\eta\in (0,1]$ and the risk can be bounded as
\begin{equation*}
\begin{split}
 \mathcal{R}(g^{\lambda}) - \mathcal{R}^{*}&\leqslant \frac{24K^{2}}{\sqrt{n\lambda^{2}}} +18(1+K\frac{D(\lambda)}{\lambda})\sqrt{\frac{log(2/\delta_{2})}{2n}}+2(1+\frac{K}{\sqrt{\lambda}})^{2} \sqrt{\frac{2log(2/\delta_{2})}{n}}\\
 &+D(\lambda)+C_{D}\,[2K_{max}\frac{1}{\lambda}\sqrt{\frac{1}{\lambda}}+\frac{\lambda}{\lambda}]\,\frac{\epsilon^{d}}{\lambda}\\
 &\leqslant O(\frac{1}{\sqrt{n\lambda^2}})\sqrt{log \frac{2}{\delta_{2}}} +O(\frac{1}{\sqrt{n\lambda}}) +
 O(\frac{1}{\sqrt{n\lambda^{2(1-\eta)}}})\sqrt{log \frac{2}{\delta_{2}}}+ O(\lambda^{\eta})+ O(\frac{\epsilon^{d}}{\lambda^{\frac{3}{2}}})
 \end{split}
\end{equation*}
\end{proposition}
\begin{proof}
Using the above result \ref{append:TECrisk} and substituting values from\ref{remark:functions} relevant to square hinge loss from remarks \ref{remark:functions}
\end{proof}

\begin{corollary}
\label{corollary:sqhingeloss}
For square hinge loss,Let $\epsilon=(\frac{1}{n})^{\frac{\mu}{2d}}$ for $0 < \mu< 1 $ and $\lambda=(\frac{1}{n})^\frac{\mu}{ 2\eta +3}$ for some $0<\eta\leqslant 1$    ,$P_{j}=\lceil3 r^{\frac{2}{d}} n^{\frac{\mu}{d}} [log (n/\delta_{1})]^{\frac{1}{d}}\rceil+1$. For some $\delta_{1} \in (0,\frac{1}{2})$ and $\delta_{2} \in (0,1)$ with probability $(1-\delta_{2})(1-2\delta_{1})$
\begin{equation}
 \mathcal{R}(g^{\lambda})- \mathcal{R}^{*} \leqslant C \sqrt{log(\frac{2}{\delta_2})} (\frac{1}{n})^{\frac{\mu\eta}{2\eta +3}}
\end{equation}
\end{corollary}

\begin{proof}
From proposition \ref{append:ratesqhinge}, we calculate this rate by substituting $\epsilon$ and $\lambda$ by their appropriate powers of $n$. Example 3.4 \citet{chen2014convergence} have derived similar result for vector i.e. single mode mode tensor with $d=1$
\end{proof}

\subsection{Rates for hinge loss}
\label{append:ratehinge}
\begin{proposition}
For  hinge loss,for $P_{j}=$ such that for each $j=\{1,2,..d\}$ the projected dimension $P_{j}=\frac{log\frac{n}{\delta_{1}}}{\epsilon^{2}}$ Then we have with probability at least $(1 - 2\delta_{1}) (1 - \delta_{2})$, for some $\eta\in (0,1]$ and some the risk can be bounded as
\begin{equation*}
\begin{split}
 \mathcal{R}(g^{\lambda}) - \mathcal{R}^{*}&\leqslant \frac{12K}{\sqrt{n\lambda}} +9\frac{(1+K)}{\sqrt{\lambda}}\sqrt{\frac{log(2/\delta_{2})}{n}}+2(1+K\frac{D(\lambda)}{\sqrt{\lambda}}) \sqrt{\frac{2log(2/\delta_{2})}{n}}\\
 &+D(\lambda)+C_{D}[2+\frac{\lambda}{\lambda}]\frac{\epsilon^{d}}{\lambda}\\
 &\leqslant O(\frac{1}{\sqrt{n\lambda}}) \sqrt{log(\frac{2}{\delta_2})} +
 O(\frac{1}{\sqrt{n\lambda^{2(1-\eta)}}})\sqrt{log(\frac{2}{\delta_2})}+ O(\lambda^{\eta})+ O(\frac{\epsilon^{d}}{\lambda})
 \end{split}
\end{equation*}
\end{proposition}

\begin{proof}
Using the above result \ref{append:TECrisk} and substituting values from\ref{remark:functions}  relevant to square hinge loss
\end{proof}

\begin{corollary}
\label{corollary:hingeloss}
For hinge loss,Let $\epsilon=(\frac{1}{n})^{\frac{\mu}{2d}}$ for $0 < \mu< 1 $ and $\lambda=(\frac{1}{n})^\frac{\mu}{ 2\eta +2}$ for some $0<\eta\leqslant 1 $,    $P_{j}=\lceil3 r^{\frac{2}{d}} n^{\frac{\mu}{d}} [log (n/\delta_{1})]^{\frac{1}{d}}\rceil+1$ , For some $\delta_{1} \in (0,\frac{1}{2})$ and $\delta_{2} \in (0,1)$ with probability $(1-\delta_{2})(1-2\delta_{1})$
\begin{equation}
 \mathcal{R}(g^{\lambda})- \mathcal{R}^{*} \leqslant C \sqrt{log(\frac{2}{\delta_2})} (\frac{1}{n})^{\frac{\mu\eta}{2\eta +2}}
\end{equation}

\end{corollary}
\begin{proof}
From proposition \ref{append:ratehinge}, we calculate this rate by substituting $\epsilon$ and $\lambda$ by their appropriate powers of $n$. Example 3.5 \citet{chen2014convergence} have derived similar result for vector i.e. single mode mode tensor with $d=1$.
\end{proof}

\subsection{Proof of theorem \ref{thm:expectation}}
\label{append:expectation}
This subsection we show that with the expected risk difference vanishes from the weaker result that  the risk difference vanishes in probability.\\

\begin{corollary}
\label{corollary:l1 conv}
For each random projector $A$, and projected mode dimension  $P_{j}=\lceil3 r^{\frac{2}{d}} n^{\frac{\mu}{d}} [log (n/\delta_{1})]^{\frac{1}{d}}\rceil+1$. If $\mathcal{R}_{\mathbf{A}}(g^{\lambda})- \mathcal{R}^{*} \to 0$ in probability, then  $$\mathbb{E}_{A}[\mathcal{R}_{\mathbf{A}}(g^{\lambda})- \mathcal{R}^{*}]\to 0$$
\end{corollary}

\begin{proof}
Suppose, $Z_{n}=\mathcal{R}_{\mathbf{A}}(g^{\lambda})- \mathcal{R}^{*}$. We can show from theorem \ref{consistencythm} that, with probability $2\delta_{1}\delta_{2}$, the following bound is true $Z_{n}>V_{1}(\delta_{2},n)+V_{2}(n)+V_{3}(\delta_{1},n,q)$. The expansion of these term can be found in theorem\ref{consistencythm}. We aim to show the $\mathbb{E}_{\mathbf{A}}(Z_{n}) \to 0$, only term of interest is $V_{3}(\delta_{1},n,q)$ as this involves error due to random projection only. Here, expectation is taken with respect to random projection measure. For  hinge loss, $q=1$ and  for hinge loss, $q=\frac{3}{2}$\\
To simplify,we can further split up as a sum of three random variables $Z_{n}=Z_{1,n}+Z_{2,n}+Z_{3,n}$, with assumption that for each $m=\{1,2,3\}$; $Z_{m,n}>V_{m}$ with corresponding probabilities.\\
Thus with probability $2\delta_{1}$,  the following holds by proposition \ref{prop:randerr} $Z_{3,n}>V_{3}(\delta_{1},n,q)=O(\frac{\epsilon^{d}}{\lambda^{q}})$. In previous expression,we note that in dependence on $\delta_{1}$ has been adjusted by choosing appropriate $P_{j}$s as discussed in \ref{discussion_epsilon}\\
Based on dependence of $\epsilon$ and $\lambda$ on sample size $n$ as chosen in corollary\ref{corollary:hingeloss} and corollary \ref{corollary:sqhingeloss}, we obtain with probability $2\delta_{1}$; $Z_{3,n}<O(\frac{1}{n})^{\frac{\mu\eta}{2\eta+2}}$  and $Z_{3,n}<O(\frac{1}{n})^{\frac{\mu\eta}{2\eta+3}}$ for hinge and square hinge loss respectively.\\
Therefore, $V_{3}$ is bounded almost surely with respect to random projection projection measure. Similarly, $V_{1}$ and $V_{2}$ are bounded almost surely with respect to random projection projection measure. The above implies that $Z_{n}$ is bounded almost surely with respect to random projection projection measure. Since it is proven in corollary \ref{corollary:hingeloss} and corollary \ref{corollary:sqhingeloss} that $\mathbf{Pr}_{\mathbf{A}}(Z_{n}=0)\to 1$. Using bound of $Z_{n}$, $\mathbb{E}_{\mathbf{A}}(Z_{n})\to 0$ by dominated convergence theorem.
\end{proof}

\end{document}